%% file: main.tex
\documentclass[sigconf]{aamas}

\usepackage{balance} %
\usepackage{decpomdps}

\newlist{questionenum}{enumerate}{1}
\setlist[questionenum]{label=\textbf{(Q\arabic*)}, ref=(Q\arabic*)}
\Crefname{question}{Question}{Questions}
\crefname{question}{question}{questions}
\crefalias{questionenumi}{question}

\setcopyright{ifaamas}
\acmConference[AAMAS '25]{Proc.\@ of the 24th International Conference
on Autonomous Agents and Multiagent Systems (AAMAS 2025)}{May 19 -- 23, 2025}
{Detroit, Michigan, USA}{A.~El~Fallah~Seghrouchni, Y.~Vorobeychik, S.~Das, A.~Nowe (eds.)}
\copyrightyear{2025}
\acmYear{2025}
\acmDOI{}
\acmPrice{}
\acmISBN{}

\acmSubmissionID{162}

\title[]{Tighter Value-Function Approximations for POMDPs}

\author{Merlijn Krale}
\affiliation{
  \institution{Radboud University}
  \city{Nijmegen}
  \country{The Netherlands}}
\email{merlijn.krale@ru.nl}

\author{Wietze Koops}
\affiliation{
    \institution{Lund University, Sweden}
    \city{}
    \country{}}
\affiliation{
    \institution{University of Copenhagen, Denmark}
    \city{}
    \country{}}
\email{wietze.koops@cs.lth.se}

\author{Sebastian Junges}
\affiliation{
  \institution{Radboud University}
  \city{Nijmegen}
  \country{The Netherlands}}
\email{sebastian.junges@ru.nl}

\author{Thiago D. Simão}
\affiliation{
  \institution{Eindhoven University of Technology}
  \city{}
  \country{The Netherlands}}
\email{t.simao@tue.nl}

\author{Nils Jansen}
\affiliation{
    \institution{Ruhr-University Bochum, Germany}
    \city{}
    \country{}
    }
\affiliation{
  \institution{Radboud University}
  \city{Nijmegen}
  \country{The Netherlands}}
\email{n.jansen@rub.de}

\begin{abstract}
Solving partially observable Markov decision processes (POMDPs) typically requires reasoning about the values of exponentially many state beliefs. 
Towards practical performance, state-of-the-art solvers use value bounds to guide this reasoning.
However, sound upper value bounds are often computationally expensive to compute, and there is a tradeoff between the tightness of such bounds and their computational cost.
This paper introduces new and provably tighter upper value bounds than the commonly used \emph{fast informed bound}.
Our empirical evaluation shows that, despite their additional computational overhead, the new upper bounds accelerate state-of-the-art POMDP solvers on a wide range of benchmarks.
\end{abstract}



\begin{document}

\pagestyle{fancy}
\fancyhead{}

\maketitle 

\section{Introduction}

Partially observable Markov decision processes (POMDPs) are a versatile modeling framework for stochastic environments where the decision maker (the agent) cannot fully observe the current state of its environment~\cite{DBLP:journals/ai/KaelblingLC98}.
Finding optimal policies for POMDPs is generally undecidable~\cite{DBLP:conf/aaai/MadaniHC99}.
Yet, in recent years, methods like POMCP~\cite{DBLP:conf/nips/SilverV10}, DESPOT~\cite{DBLP:journals/jair/YeSHL17}, and AdaOPS~\cite{DBLP:conf/nips/WuYZYLLH21} have been able to find policies for increasingly large POMDPs.

Although such methods often provide a (statistical) \emph{lower bound} on the value of the policy, they are typically unable to find an \emph{upper bound} on the optimal value. 
Such further certification of the quality of a policy may be essential for safety-critical problems.
For example, planning medical treatments~\cite{DBLP:journals/artmed/HauskrechtF00}, scheduling infrastructure maintenance~\cite{morato2022managing,morato2022optimal} or computing safe flight paths~\cite{temizer2010collision} require us not only to know how well our policy will perform, but also that we cannot (reasonably) do any better.

So-called $\epsilon$-optimal solvers such as \SARSOP~\cite{DBLP:conf/rss/KurniawatiHL08} and HSVI~\cite{DBLP:conf/uai/SmithS05} compute both a policy and an upper bound.
These algorithms make use of heuristic search to find good policies quickly.
However, they often struggle to find upper bounds that are reasonably tight, since this requires reasoning over \emph{all} possible policies.

Both HSVI and \SARSOP use the \emph{fast informed bound}, or \FIB~\cite{DBLP:journals/jair/Hauskrecht00}, to initialize their upper bound computations.
Intuitively, \FIB computes values in a simplified POMDP, where the agent fully observes the state of the environment with a delay of one time step. 
However, these bounds are often loose in practice, while tighter upper bounds could improve the performance of $\epsilon$-optimal solvers.

We contribute three different methods to obtain bounds that exhibit varying levels of tightness and computational overhead. 

We first introduce the \emph{tighter informed bound}~(\BIB) as an alternative for \FIB.
Intuitively, \BIB uses a delay of two time steps rather than one time step. \BIB can be computed using value iteration, as employed by \cite{bellmann1957,DBLP:books/wi/Puterman94}, on all \emph{one-step beliefs}, that is,  beliefs the agent can have one time step after knowing the state. 
These precomputations are more expensive than for \FIB, but allow to compute a bound for any belief at the same computational cost as \FIB.
However, we show that increasing the delay further would significantly increase these computational costs.

Closer inspection of \BIB shows that it expresses posterior beliefs of the agent as a convex combination of one-step beliefs. 
However, choosing \emph{different} combinations may further tighten the bound. 
The \emph{optimized tighter informed bound}~(\OBIB) uses the convex combinations that yield the tightest possible bound. 
However, finding this convex combination requires solving a linear program for each posterior belief in each iteration step, which is usually too expensive. 
Instead, the \emph{entropy-based tighter informed bound}~(\EBIB) heuristically chooses a single combination for each posterior belief by maximizing the weighted entropy of the chosen one-step beliefs. 
This combination is reused for each iteration, thus greatly reducing computational cost.

Empirically, \BIB and \EBIB provide better bounds than \FIB on a large range of benchmarks with reasonable computational cost. 
To test the practical relevance of our bounds, we adapt the offline state-of-the-art solver \SARSOP \cite{DBLP:conf/rss/KurniawatiHL08} to use our upper bounds as initialization.
With this alteration, \SARSOP finds tighter optimality bounds more quickly on a wide range of benchmarks, which means the additional computational overhead of our bounds is compensated by a speedup in convergence.
Moreover, this positive effect grows as the discount factor increases. 

\paragraph{Contributions.} To summarize, our \textbf{main contributions} are introducing three novel bounds for POMDPs, namely \BIB, \EBIB, and \OBIB.
These bounds both theoretically and empirically improve prior methods. 
Moreover, integrating these novel bounds with the state-of-the-art $\epsilon$-optimal solver SARSOP~\cite{DBLP:conf/rss/KurniawatiHL08} leads to significant speedups and smaller optimality gaps. 

\section{Problem Setting}
\label{sec:background}

In this section, we provide a formal definition of our problem setting in order to formalize our problem statement.
We first introduce some basic notation: $\distr{X}$ denotes the set of probability distributions over a finite set $X$.
Given a function $F\colon X \rightarrow \distr{Y}$ and elements $x \in X$, $y\in Y$, $F(\cdot \midd x)$ denotes the conditional probability distribution over $Y$ given $x$, $F(y \midd x)$ the probability of element $y$ given $x$, and $y \sim F(x)$ an element $y$ randomly sampled from $F(x)$.

\subsubsection*{POMDPs}
An (infinite-horizon, discounted) \emph{partially observable Markov decision process} (POMDP)~\cite{DBLP:journals/ai/KaelblingLC98,DBLP:books/sp/12/Spaan12} is defined as a tuple $\model=\tuple{\states,\actions,\transitions,\observations, \obsfun,\rewards,\discount}$, with $\tuple{\states,\actions,\transitions,\rewards, \discount}$ an  MDP~\cite{DBLP:books/wi/Puterman94} with a finite set of \emph{states}~$\states$, a finite set of \emph{actions}~$\actions$, a \emph{transition function}~$\transitions \colon \states \times \actions \rightarrow \distr{\states}$, a \emph{reward function}~$\rewards \colon \states \times \actions \rightarrow \mathbb{R}$, and a \emph{discount factor} $\discount \in (0,1)$.
Additionally, $\observations$ is a finite set of \emph{observations} and $\obsfun \colon \actions \times \states \rightarrow \distr{\observations}$ is the \emph{observation~function}.

A POMDP models the interaction between a stochastic environment and an agent. Let  $b_0 \in \distr{\states}$ be the fixed \emph{initial distribution} (aka \emph{initial belief}).  
The \emph{initial state} $s_0$ of the environment is sampled from $b_0$.
At each time step $t$, the agent picks an action $\act_t\in \actions$. 
As a result, the environment transitions to a new state $s_{t+1} \sim \transitions(\cdot \midd s_t, \act_t)$ and returns a reward $\reward_t = \rewards(s_t,\act_t)$.
However, unlike for MDPs, the agent does not observe the state $s_{t+1}$, but instead receives an observation $\obs_{t+1} \sim \obsfun(\cdot \midd \act_t, s_{t+1})$.
In general, agents make decisions based on their history $(b_0, \act_0, \obs_1, \dots, \act_t, \obs_{t+1})$.
As shown by \citet{astrom1965optimal}, this history can be summarized by a \emph{belief} $b_t \in \distr{\states}$.
Therefore, we can assume that the agent chooses actions according to a (deterministic) belief-based policy $\pi \colon \distr{\states} \rightarrow \actions$.
Given a policy $\pi$ and an initial belief $b$, we define the \emph{value} as the expected discounted return over an infinite horizon:
\[
    \E_\pi\left[\sum_{t=0}^{\infty} \discount^t r_t | s_0 \sim b\right].
\]
The agent aims to maximize the value for the initial belief $b_0$.

\subsubsection*{Probabilities}
We now introduce additional notation that will be used throughout this paper.
For any belief $b  \in \distr{\states}$, let $R(b,a) = \sum_{s \in S} b(s) R(s,a)$ be the expected reward of action $a$ in belief $b$ and let $T(s' \midd b,a) = \sum_{s \in S} b(s) T(s' \midd s,a)$ be the probability of transitioning to state $s'$ when taking action $a$ in belief $b$.

We define a shorthand for four probabilities.
Given a state $s$ and an action $\act$, the probability of transitioning to state $s'$ and observing $\obs$ is denoted by 
\[
\Pr(s', \obs \midd s, \act) = \obsfun(\obs \midd \act, s') \transitions(s' \midd s, \act),
\]
while the probability of observing $\obs$ is denoted by
\[
\Pr(\obs \midd s, \act) = \sum_{s' \in \states} \Pr(s', \obs \midd s, \act).
\]
Given a belief $b$ and action $\act$, we denote the probability of transitioning to $s'$ and observing $\obs$ by
\[ \Pr(s', \obs \midd b, \act) = \sum_{s \in \states} [ b(s) \Pr(s', \obs \midd s, \act) ],
\]
while the probability of observing~$\obs$ is given by
\[
\Pr(\obs \midd b, \act) = \sum_{s' \in \states} \Pr(s', \obs \midd b, \act).
\]

\subsubsection*{Beliefs}
We also define notation for specific beliefs.
For any $s \in \states$, let the \emph{unit belief}~$\bel_s$ be the belief such that $\bel_s(s)=1$ (and hence $\bel_s(s') = 0$ for $s' \neq s$). 
Let $\mathcal{B}_\states = \{\bel_s \mid s \in \states\}$ be the set of all unit beliefs. 
If $\Pr(\obs \midd b, \act) >0$,  $\bel_{b,a,o}$ is the belief after taking action $a$ and observing $o$ from belief $b$, i.e.:\footnote{
For conciseness, we assume beliefs $\bel_{b,a,o}$ with $\Pr(\obs \midd b, \act) = 0$ are arbitrarily defined, and that sums over observations consider only those observations that occur with non-zero probability.
}
\begin{equation} \label{eq:bbao}
    \bel_{b, \act, \obs}(s') = \Pr(s' \midd b, \act, \obs) =  \frac{\sum_{s\in \states} b(s) \Pr(s', \obs \midd s, \act)}{\Pr(\obs \midd b, \act)}.
\end{equation}
Further, we denote a \emph{one-step beliefs} as $\bel_{s,\act,\obs} = \bel_{\bel_s,\act,\obs}$, which denotes a belief reached from a unit belief $\bel_s$ (i.e., a belief where the agent knows the state) in a single time step after executing $a$ and observing $o$.
We define $\Bsao$ as the (finite) set containing all one-step beliefs and the initial belief $b_0$, i.e.,
\begin{equation}
\label{eq:Bsao}
    \Bsao = \{ \bel_{s,\act,\obs} \mid s\in\states, \act \in \actions, \obs \in \observations, \Pr(\obs \midd s, \act) > 0 \} \cup \{b_0\}.
\end{equation}
See \cref{ex:guessing} for a concrete example of sets $\Bs$ and $\Bsao$.
We note that every reachable belief (except possibly $b_0$) can be written as a convex combination of one-step beliefs.
Hence, all reachable beliefs can be written as a convex combination of beliefs in $\Bsao$.

\subsubsection*{$Q$-values}
Lastly, to reason about the decision-making process of an agent, we define the $Q$-value function $Q \colon \distr{\states} \times \actions \rightarrow \mathbb{R}$ as the value for a given belief-action pair. 
Let $\setQ$ be the set of all functions $Q \colon \distr{\states} \times \actions \rightarrow \mathbb{R}$.
The $Q$-value function corresponding to an optimal policy can be given as the (unique) fixed point of the \emph{Bellman operator} $H_\POMDP \colon \setQ \rightarrow \setQ$~\cite{DBLP:journals/ior/Sondik78}:
\begin{equation}
    \label{eq:QPOMDP}
    \! H_{\POMDP}Q(b,a) = R(b,a) + \gamma \!\!\sum_{o \in \observations} \! \Pr(o \midd b,a) \! \max_{\act' \in \actions} \!  Q(\bel_{b,a,o},a').
\end{equation}

With our problem setting defined, we formalize our problem statement as follows:
\begin{problemstatement} \,\!
    Find tractable methods of computing tight overapproximations (or \emph{bounds}) of the $Q$-value function for POMDPs to improve the performance of $\epsilon$-optimal solvers.
\end{problemstatement}

\section{Prior Methods}

In this section, we describe the baseline methods of finding upper bounds for POMDPs using the notation introduced in \cref{sec:background}. 
We discuss the \emph{fast informed bound}~(FIB)~\cite{DBLP:journals/jair/Hauskrecht00}, but define it using $Q$-functions.
Then, we recall \emph{point set bounds}~\cite{DBLP:conf/ijcai/PineauGT03} and show how FIB can be interpreted as a point set bound.
Finally, we briefly review how upper bounds are used in state-of-the-art solver \SARSOP.

First, we introduce the \custom POMDP (\cref{fig:NoOpEnv}), which we will use as a running example to illustrate the various upper bounds.

\begin{figure}[tb]
    \centering
    \includegraphics[width=0.5\columnwidth]{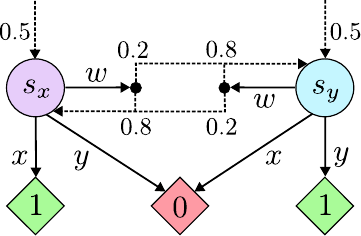}
    \caption{Visualisation of the \custom POMDP.}
    \label{fig:NoOpEnv}
    \Description{TODO}
\end{figure}

\begin{example}[\custom]
\label{ex:guessing}
In the \custom POMDP (\cref{fig:NoOpEnv}), the agent starts in an initial belief $b_0$, with $b_0(s_x) = b_0(s_y) = 0.5$.
From here, the agent can \emph{guess} in which state it is by taking actions $x$ or $y$, which both lead to a terminal state $s_\text{sink}$ (not depicted), which yields a reward of $1$ if the state is guessed correctly and 0 otherwise.
Alternatively, the agent can execute the \emph{waiting} action $w$, which has a probability of $0.2$ to transition to the other state and $0.8$ to stay in the same state.
All state-action pairs yield the same observation (denoted $\bot$), and we assume a discount factor $\gamma \in [0.9, 1)$.
Since taking the waiting action $w$ does not change the agent's belief and yields no reward, it is intuitively easy to see that an optimal policy is to pick action $x$ (or action $y$), yielding an expected reward of $0.5$.
\end{example}

In \custom, the sets of unit- and one-step beliefs are written as:
\begin{equation*}
\begin{aligned}
    \Bs = \{ \bel_{s_x} \ \bel_{s_y} \ & \bel_{s_\text{sink}} \} \\
    \Bsao = \{ b_0 \} \cup \{ & \bel_{s_x, w, \bot} \ \bel_{s_y, w, \bot} \ \bel_{s_x, x, \bot} \\
    & \bel_{s_y, x, \bot} \ \bel_{s_x, y, \bot} \ \bel_{s_y, y, \bot} \ \bel_{s_\text{sink}, w, \bot}\}.
\end{aligned}
\end{equation*}
We note that many beliefs in $\Bsao$ describe the same state distribution: in fact, the last 5 elements are all equal to $\bel_{s_\text{sink}}$.
However, throughout this paper, we will regard such beliefs as distinct members of this set for notational simplicity.

\subsection{Fast Informed Bound (\FIB)}
\label{sec:priorbounds}

A common method of over-approximating the value of a POMDP is to (partially) ignore the effect of partial observability.
The most straightforward example of this is the \emph{QMDP bound}~\cite{DBLP:conf/icml/LittmanCK95}, which intuitively corresponds to the assumption that agents can fully observe their state in the future.
We can define this as follows:
\begin{definition}
    $\qmdp$ is the fixed point of the operator $H_\text{QMDP}$:
    \begin{equation}
    \begin{aligned}
        H_{\text{MDP}} Q(b,a) = R + \gamma\sum_{s'\in\mathcal S}[\Pr(s'|b,a)\max_{a'\in\mathcal A}Q(b_{s'},a')].
    \end{aligned}
    \end{equation}
\end{definition}
\noindent To further tighten this bound, \textbf{the \emph{fast informed bound} (\FIB)~\cite{DBLP:journals/jair/Hauskrecht00} assumes an agent fully observes the current and future states with a delay of 1 time step}. 
More precisely, we define $\qfib$, the $Q$-value function for this bound, as follows:
\begin{definition}
    $\qfib$ is \emph{the} fixed point of the operator~$H_\FIB$:
    \begin{equation} \label{eq:FIB}
    \begin{aligned}
        H_{\FIB}Q&(b,a) =  R(b,a) \\ &+ \gamma \cdot \sum_{o \in \observations} ~\max_{a' \in \actions} ~\sum_{s' \in \states}  \left[ \Pr(\obs, s' \midd b,a) Q(\bel_{s'},\act') \right].
    \end{aligned}
    \end{equation}
\end{definition}
\noindent \cref{ap:proofs} provides a proof that this fixed point exists and is unique (based on the original proof from \citet{DBLP:journals/jair/Hauskrecht00}).

In \cref{eq:FIB}, the next action $a'$ is picked independently of the next state $s'$ but must depend only on the current belief~$b$ and received observation $o$.
However, for all future time steps, we use $Q$-values computed for the unit belief $\bel_{s'}$, i.e., as if~$s'$ is revealed.
Thus, the formula matches the intuitive description of full observability delayed by one time step.
In contrast to \cref{eq:QPOMDP}, $H_\FIB$ depends only on the $Q$-values of the (finite) set of beliefs $\bel_s \in \Bs$.
Thus, the value of $\qfib$ for \emph{any} belief $b$ can be computed efficiently by (approximately) computing the fixed point for beliefs in $\Bs$. 
Both the QMDP bound and \FIB are commonly used in POMDP literature due to their tractability but tend to be loose.

\subsubsection*{Running example.}
Recall the \custom POMDP.
Under the QMDP assumption, taking action $w$ would fully reveal the agent's state.
In that case, an agent can always guess correctly after taking action $w$, which yields an expected value of $\gamma$.
Similarly, under the \FIB assumption, taking action $w$ would fully reveal the agent's previous state.
The probability of still being in this state after this action is $0.8$.
Thus, taking action $w$ and guessing the revealed initial state yields an expected return of $0.8\gamma$.
Both are strict overapproximations of the optimal value $0.5$ of the POMDP, and both incorrectly give higher $Q$-values for action $w$ than for $x$ or $y$.

\subsection{Point Set Bounds}
\label{sec:pointsets}

To compute tighter approximations than \FIB, we consider a general value bound that uses \emph{point sets}~\cite{DBLP:conf/ijcai/PineauGT03}: sets of beliefs with known upper bounds.
To make the connection with our own method more clear, we define them using our own (non-standard) notation.
We start by defining a \emph{weight function} as follows:
\begin{definition} \label{def:weightf}
    Let $b \in \distr{\states}$ be a belief and let $\mathcal{B} \subseteq \distr{\states}$ be a point set.
    A \emph{weight function} $w \colon \mathcal{B} \rightarrow \mathbb{R}_{\geq 0}$ is any function satisfying $b(s) = \sum_{b' \in \mathcal{B}} w(b') b'(s)$ for all $s \in \states$.
    $\weightset_{\mathcal{B}, b}$ denotes the set of all possible weight functions for belief $b$ given point set~$\mathcal{B}$.
    \footnote{
    $\weightset_{\mathcal{B}, b}$ may be empty if $b$ does not lie in the convex hull of $\mathcal{B}$.
    }
\end{definition}
\noindent We can use weight functions to compute upper bounds as follows:
\begin{theorem}[Point set bound]
\label{thm:Pointset}
    Given a belief $b$, a \emph{point set} $\mathcal{B}$, and a function   
    $Q \colon \mathcal{B} \times \actions \rightarrow \mathbb{R}$ which over-approximates the $\qpomdp$-values of all beliefs-action pairs $(b',a) \in \mathcal{B} \times \actions$.
    Then, any weight function $w \in \weightset_{\mathcal{B}, b}$ gives an upper bound on the value of $b$:
    \begin{equation} \label{eq:PSA}
        \qpomdp(b,a) \leq \Qinner{}{w}{t}{a} := \sum_{b' \in \mathcal{B}} w(b') Q(b',a).
    \end{equation}
\end{theorem}  
This theorem follows directly from the convexity of the value function for POMDPs~\cite{DBLP:journals/ior/Sondik78}.
To understand how \cref{thm:Pointset} is used implicitly by \FIB,
consider using point set $\mathcal{B}_\states$ and the weight functions $w_{b,a,o}(\bel_{s'}) = \frac{\Pr(\obs, s' \midd b,a)}{\Pr(\obs \midd b,a)} $.
In that case, we find:
\begin{equation}
    \! H_{\FIB}Q  (b,a) = R(b,a) \!+\! \gamma \!\!\sum_{\obs \in \observations} \max_{\act' \in \actions}\!\left[\Pr(\obs \midd b,a) \Qinner{\FIB}{w_{b,a,o}}{t-1}{a'}\right],\!
\end{equation}
with $\Qinnershort{\FIB}$ the weighted sum over values of $\qfib$ as defined in \cref{eq:PSA}.

Given a point set $\mathcal{B}$ and belief $b$, the \emph{tightest} upper bound we can compute using \cref{thm:Pointset} is found using the linear program~(LP): $\min_{w \in \weightset_{\mathcal{B},b}} \Qinner{}{w}{t}{a}.$
However, POMDP solvers often need to compute upper bounds for many beliefs using large point sets, in which case this method is computationally expensive.
Thus, instead of solving these LPs exactly, solvers may approximate their outcome instead.
One such approximation method is the \emph{sawtooth bound}~\cite{DBLP:journals/jair/Hauskrecht00}.
This bound is based on the observation that, for point sets of the form $\{b'\} \cup \Bs$, an upper bound can be computed in only $\bigO(|\states|)$ time.
Thus, if $\Bs \subseteq \mathcal{B}$, we can compute such a bound for all beliefs $b' \in \mathcal{B}$ and take their minimum.
This takes only $\bigO(|\states| |\mathcal{B}|)$ time, but still yields tight bounds in practice.
We refer to \citet{Kochenderfer2022} for a detailed implementation of the sawtooth bound.

\subsection{Using Bounds in Point-Based Solvers}
\label{sec:SARSOP}

Point set bounds are an important component of \emph{point-based solvers}, a type of algorithm that uses a finite set of beliefs to compute both upper- and lower bounds on the value of a POMDP.
Early methods use predefined sets of beliefs to cover the entire belief space evenly~\cite{DBLP:journals/ior/Lovejoy91,DBLP:conf/ijcai/ZhouH01,DBLP:conf/icml/Bonet02}, but these methods typically scale poorly to large POMDPs.
Instead, state-of-the-art algorithms such as HSVI~\cite{DBLP:conf/uai/SmithS04, DBLP:conf/uai/SmithS05} and \SARSOP~\cite{DBLP:conf/rss/KurniawatiHL08} use \emph{heuristic search} to find beliefs that closely resemble those encountered by an optimal policy, which is sufficient for finding $\epsilon$-optimal solutions~\cite{DBLP:conf/rss/KurniawatiHL08}.

\SARSOP~\cite{DBLP:conf/rss/KurniawatiHL08} is a state-of-the-art point-based solver that uses a variant of value iteration~\cite{DBLP:journals/ior/SmallwoodS73} to compute lower bounds and the sawtooth bound for upper bounds.
The latter requires precomputing value bounds for the set of unit beliefs $\Bs$, which is traditionally done using \FIB. %
The next section proposes methods of computing tighter bounds for this set in tractable time.

\section{Introducing tighter bounds}
\label{sec:heuristics}

 In this section, we introduce three novel bounds on the value function $\qpomdp$, which are tighter than \FIB.

\subsection{Tighter Informed Bound (\BIB)}

\begin{algorithm}[tb]
\caption{\textsc{Precomputations for \BIB, \OBIB and \EBIB}}
\label{alg:BIB}
\begin{algorithmic}
\State Compute all (unique) beliefs in $\Bsao$ \Comment{\cref{eq:Bsao}}
\For{$b,a  \in \Bsao \times \actions$}
    \State $Q'(b,\act) \gets Q_\FIB(b,\act)$ \Comment{\cref{eq:FIB}}
    \State $\forall \obs \in \observations$, precompute $\went_{b,a,o}$ \Comment{\cref{eq:w_ent}, only for \EBIB}
\EndFor
\For{$h$ iterations}
    \State $Q \leftarrow Q'$
    \For{$b \in \Bsao, \act \in \actions$}
        \State $Q'(b,\act) \leftarrow HQ(b,\act)$ \Comment{\cref{eq:BIB}, \cref{eq:OBIB} or \cref{eq:EBIB}}
    \EndFor
    \State \textbf{if } $\frac{\gamma}{1-\gamma} \left\|\frac{Q' - Q}{Q} \right\|_\infty < \epsilon$ \textbf{ then break} \Comment{Required precision reached}
\EndFor
\State \textbf{return} $Q'$
\end{algorithmic}
\end{algorithm}

Firstly, we propose an extension of \FIB that extends the delay at which the full state is observed.
More precisely, we define the \textbf{\emph{tighter informed bound} (\BIB), which assumes an agent fully observes the current and future states with a delay of 2 time steps.}
We define the corresponding $Q$-value function, $\qbib$, as:
\begin{definition}
    $\qbib$ is \emph{the} fixed point of the operator~$H_\BIB$:
    \begin{equation} \label{eq:BIB}
    \begin{aligned}
        H_{\BIB}&Q(b, \act) = R(b,\act) \\ &+ \gamma \! \sum_{\obs \in \observations}  \max\limits_{\act' \in \actions} \sum_{s \in \states} \! \left[ b(s) \Pr(\obs \midd s, \act) Q(\bel_{s,\act,\obs}, \act') \right].
    \end{aligned}
    \end{equation}  
\end{definition}
\noindent We provide a proof that this unique fixed point exists in \cref{ap:proofs}.\footnote{
This follows from showing $H_\BIB$ is a contraction mapping with Lipschitz constant $\gamma < 1$ and using Banach's fixed point theorem~\cite{banach1922operations}.}

Recall from \cref{eq:FIB} that for \FIB, we can compute $Q$-values for any belief using only $Q$-values of unit beliefs $\bel_{s'}$. 
This intuitively corresponds to state~$s'$ being observed with a delay of 1 time step.
In contrast, in \cref{eq:BIB}, we use $Q$-values of one-step beliefs $\bel_{s,a,o}$, which corresponds to state $s$ being revealed with a delay of $2$ time steps and thus aligns with the intuitive definition of $\BIB$.

Next, we highlight some important properties of $\qbib$:
\begin{theorem}
\label{thm:Bib_props}
    $\qbib$ has the following properties:
    \begin{enumerate}
        \item \textbf{Soundness:} $\forall b \in \distr{\states}, a \in \actions \colon \qbib(b,a) \geq \qpomdp(b,a)$;
        \item \textbf{Tightness:} $\forall b \in \distr{\states}, a \in \actions \colon \qfib(b,a) \geq$ \qbib(b,a).
    \end{enumerate}
\end{theorem}

Full proofs of \cref{thm:Bib_props} are provided in \cref{ap:fib_proofs}.
Intuitively, we recall that \BIB and \FIB correspond with the agent observing the state with a delay, which gives them additional information.
Since this information can only help the agent, $\qbib$ is a sound upper bound (1).
Moreover, since our model is Markovian and the delay of \FIB is lower than that of \BIB, the additional information of \FIB is at least as useful.
Thus, for any belief-action pair, $\qbib$ is never higher than $\qfib$ (2).

\subsubsection*{Running example.}
To provide some intuition on the tightness of \BIB, we recall the \custom POMDP.
Under the \BIB assumption, taking action $w$ twice lets an agent observe its initial state.
The probability of still being in this state after these actions is $0.8^2+0.2^2=0.68$.
Thus, taking action $w$ twice and guessing the revealed initial state yields an expected return of $0.68\gamma^2$.
This is a strict overapproximation of the optimal value of the POMDP, which is $0.5$, but significantly tighter than the bound of $0.8\gamma$ found by \FIB.

\subsubsection*{Complexity analysis.}
$\qbib$ can be computed up to an arbitrary precision using value iteration, as shown in \cref{alg:BIB}.
\footnote{When performing $h$ iterations, the computed bound for any belief-action pair is at most $\max_{(s,a) \in \states \times \actions} \frac{1}{1-\gamma} R(s,a)$ away from the fixed point.
For an explanation of the precision parameter, see \cref{ap:proofs}.
}
\cref{table:complexity} shows the computational complexity of computing $\qbib$ and $\qfib$ using such methods.
The computational costs for a single Bellman operation are equal for \FIB and \BIB, which also means the online computational costs are the same.
Precomputations for \BIB are a factor $\bigO(\nicefrac{|\Bsao|}{|\states|}) \in \bigO(|\actions||\observations|)$ more expensive than for \FIB.
The empirical evaluation shows that this is often manageable in practice~(\cref{sec:experiments}).

\subsubsection*{Further increasing delays.}
One method of computing even tighter bounds is to consider even longer observation delays.
However, increasing the observation delay also means we need to consider a larger set of beliefs.
Thus, as shown in \cref{ap:3stepdelay}, a delay of 3 time steps already yields a computational complexity of $\bigO(|\states|^2|\actions|^5|\observations|^4h)$, which is a factor $\bigO(|\actions|^2|\observations|^2)$ larger than for \BIB.
It may be possible to efficiently compute such bounds regardless, but we will not consider this line of research here.
Instead, we focus on methods that can tighten our bounds without further increasing the delay.

\begin{table}[tb]
\caption{
Computational complexities of different bounds, assuming precomputations are performed using \cref{alg:BIB} with $h$ iterations.
$L$ denotes the computational complexity of solving an LP with $|\Bsao|$ variables and $|\states|$ constraints.
}
\label{table:complexity}
\centering
\input{Tables/Complexities}
\end{table}

\subsection{Optimized Tighter Informed Bound (OTIB)}

Like for \FIB, we notice \BIB can be rewritten using \cref{thm:Pointset} with point set $\Bsao$, as follows:
\begin{equation*}
    H_{\BIB}Q(b, \act)  = R(b,\act) + \gamma \sum_{\obs \in \observations}  \max\limits_{\act' \in \actions} \left[ \Pr(\obs \midd b, \act)\Qinner{\BIB}{w_{b,\act,\obs}}{t{-}1}{\act}\right],
\end{equation*}
where we use the following weight function:
\[
w_{b,\act,\obs}(\bel_{s,\act,\obs}) = \frac{b(s) \Pr( \obs \midd s, \act)}{\Pr(\obs \midd b, \act)}.
\]
In contrast to \FIB, however, these weights are not necessarily unique, and \cref{thm:Pointset} tells us \emph{any} weight that represents our belief gives a viable upper bound.
Thus, we define the
\textbf{\emph{optimized tighter informed bound} (\OBIB), which assumes the value for future beliefs is equal to the minimal point set bound~(\cref{thm:Pointset}) using point set $\Bsao$.}
We define the corresponding $Q$-value function as follows:
\begin{definition} \label{def:wbib}
    Write $\weightset_{b,\act,\obs} = \weightset_{\Bsao, \bel_{b,\act,\obs}}$. 
    Then, $\qobib$ is \emph{the} fixed point of the operator $H_\OBIB$:
    \begin{equation} \label{eq:OBIB}
    \begin{aligned}
        H_\OBIB &Q(b,\act)  = R(b,\act) \\ &+ \gamma \!\sum_{\obs \in \observations} \max_{\act' \in \actions} [\Pr(\obs \midd b, \act) \!\! \min_{w \in \weightset_{b,a,o}} \!\!\Qinner{\OBIB}{w}{t{-}1}{\act'}].
    \end{aligned}
    \end{equation}
\end{definition}
\noindent
We note that $w_{b,\act,\obs} \in \weightset_{b,\act,\obs}$, which means the minimization in \cref{eq:OBIB} always has a feasible solution.
A full proof of the existence and uniqueness of $\qobib$ is provided in \cref{ap:proofs}.
We highlight a number of properties of $\qobib$:

\begin{theorem}
    $\qobib$ has the following properties:
    \begin{enumerate}
        \item \textbf{Soundness:} $\forall b \nsp\in\nsp\distr{\states}, a \nsp\in\nsp \actions \colon \qobib(b,a) \nsp\geq\nsp \qpomdp(b,a)$;
        \item \textbf{Tightness:} $\forall b \in \distr{\states}, a \in \actions \colon \qbib(b,a) \geq \qobib(b,a) $.
    \end{enumerate}
\end{theorem}

Proofs are provided in \cref{ap:proofs}.
Soundness follows as for \cref{thm:Bib_props}, tightness from the observation that $w_{b,\act,\obs} \in \weightset_{b,\act,\obs}$.

\subsubsection*{Running example.}
We consider the \custom POMDP (\cref{ex:guessing}).
Under the \OBIB assumption, the belief after taking action $w$ can be expressed using any weight function in $\weightset_{b_0, w, \bot}$.
In particular, since $b_0 \in \Bsao$, one valid choice uses weight 1 for $b_0$ and 0 for all others.
In that case, the action is suboptimal (with value $0.5\gamma$), and the \OBIB bound corresponds with the real value $0.5$.

\subsubsection*{Complexity analysis.}
As for \BIB, $\qobib$ can be approximated using \cref{alg:BIB}.
\OBIB and \BIB use the same point set and thus require the same amount of Bellman operations per iteration.
However, a single Bellman operation for \OBIB is significantly more expensive since it requires solving an LP with at most $|\Bsao|$ variables and $|\states|$ constraints.
This yields the computational complexities shown in \cref{table:complexity}, where $L$ denotes the complexity of solving such an LP.
These computation costs are typically too high for practical use.

\begin{table*}[tb]
\caption{
Upper bounds and computation times of different methods on a number of POMDP benchmarks.
If a bound has not converged within 1200s, we report the last computed bound and denote computation time as TO.
The tightest bounds are bolded.
We include lower bounds computed by \SARSOP as a proxy for the optimal value, with $\epsilon$-optimal values underlined.
}
\label{table:bounds}
\resizebox{\textwidth}{!}{
\centering
\input{Tables/Eval_Bounds}

}
\end{table*}

\subsection{Entropy-based Tighter Informed Bound (\EBIB): A Heuristic Approach}

To reduce the complexity of the precomputations of \OBIB, we consider using a single weight for each belief that we reuse for all iterations.
More precisely, we approximate the worst-case weights by those that maximize the \emph{weighted entropy}.
This gives higher weights to more uncertain beliefs, which should intuitively give a tighter bound.
To formalize this, we first define the \emph{maximal entropy weight function} for a belief $\bel_{b,\act,\obs}$ as follows:
\begin{equation}
\label{eq:w_ent}
    \went_{b, \act, \obs} \in \argmax\limits_{w \in \weightset_{b,a,o}} \sum_{b' \in \Bs} H(b')w(b').
\end{equation}
\noindent This equation always has a feasible solution, since $w_{b,a,o} \in \weightset_{b,a,o}$.
Then, \textbf{the entropy-based tighter informed bound (\EBIB) assumes the value for future beliefs is equal to the point set bound (\cref{thm:Pointset}) using point set $\Bsao$ and maximal entropy weight functions (\cref{eq:w_ent})}.
We define the corresponding $Q$-value function as follows:
\begin{definition}    
$\qebib$ is \emph{the} fixed point of the operator $H_\EBIB$:
    \begin{equation} \label{eq:EBIB}
    \begin{aligned}
        H_\EBIB &Q(b,\act)  = R(b,\act) \\ &+ \gamma \!\sum_{\obs \in \observations} \max_{\act' \in \actions} [\Pr(\obs \midd b, \act) \Qinner{}{\went_{b,\act,\obs}}{t{-}1}{\act'}].
    \end{aligned}
    \end{equation}
\end{definition}
As for the other bounds, we provide a proof that this unique fixed point exists in \cref{ap:proofs}.
We highlight the following properties of $\qebib$:

\begin{theorem}
\label{thm:QEBIB}
    $\qebib$ has the following properties:
    \begin{enumerate}
        \item \textbf{Soundness:} $\forall b \nsp\in\nsp \distr{\states}, a \nsp\in\nsp \actions \colon \qebib(b,a) \nsp\geq\nsp \qpomdp(b,a)$;
        \item \textbf{Tightness:} $\forall b \in \distr{\states}, a \in \actions \colon \qfib(b,a) \geq \qebib(b,a)$.
    \end{enumerate}
\end{theorem}
The proof for \cref{thm:QEBIB} is provided in \cref{ap:proofs} and follows the same intuition as those of \cref{thm:Bib_props}.
In contrast to \OBIB, we note that \emph{\EBIB does not necessarily provide tighter bounds than \BIB}, since there is no guarantee $\hat{w}_{b,a,o}$ yields tighter bounds than $w_{b,a,o}$.
In practice, however, we find \EBIB is at least as tight as \BIB, and sometimes (significantly) tighter.

\subsubsection*{Running example.}
Consider the \custom environment.
Under the \EBIB assumption, the value of taking action $w$ is approximated using the maximal entropy weight function $\hat{w}_{b_0,w,\bot}$.
Since $b_0$ is the belief with the largest entropy in $\Bsao$, this weight function is the weight function defined by $w(b)=1$ if $b=b_0$, and $w(b) = 0$ otherwise.
Thus, \EBIB and \OBIB find the same optimal bound.

\subsubsection*{Complexity analysis.}
As for our other proposed bounds, $\qebib$ can be approximated using \cref{alg:BIB}, with corresponding complexities shown in \cref{table:complexity}.
The computational complexity of a single Bellman operation is a factor $\bigO(|\actions|)$ smaller for \EBIB as compared to \OBIB, since the same weight can be used for each next action $a'$.
Moreover, since these weights can be reused at each iteration, the complexity for precomputations is significantly lower as well.
In practice, the number of beliefs $b'$ with $w_{b,a,o}(b') > 0$ is often much smaller than $|\Bsao|$, in which case the iterations take significantly less time than the complexity bound suggests.

\section{Empirical Evaluation}
\label{sec:experiments}

In this section, we empirically evaluate the proposed bounds: \BIB, \EBIB, and \OBIB. 
We address the following questions:

\begin{questionenum}
    \item \label{q:tightness} \textbf{Bounds tightness.}
        How do the proposed bounds compare to each other and prior bounds such as \FIB? How close are these bounds to the optimal value?
    \item \label{q:overhead} \textbf{Computational cost.}
        What is the computational cost of these bounds? How do they scale with the POMDP size?
    \item \label{q:integration} \textbf{Benefits for \SARSOP.}
        Can these bounds improve the performance of POMDP solvers such as \SARSOP?
    \item \label{q:discount} \textbf{Discount dependency.}
        How does the effect of using these bounds in \SARSOP depend on the discount factor?
\end{questionenum}

\paragraph{Implementation \& Baselines.}
We implement \cref{alg:BIB} within the \emph{POMDPs.jl} framework~\cite{egorov2017pomdps}, and extend the native Julia implementation of \SARSOP~\cite{DBLP:conf/rss/KurniawatiHL08} to use our bounds as initialization.
As discussed in \cref{sec:SARSOP}, we replace the bounds for $\Bs$ with those computed by our methods.
Unless stated otherwise, we use discount factor $\gamma = 0.95$ and compute bounds using relative precision $\epsilon = 10^{-3}$ and $h=250$ maximum iterations. 
By initializing the bound calculations with a (looser) upper bound, we ensure the result is valid despite the finite number of iterations.
\cref{ap:experiments} provides further details. 
All code and data is available at \url{https://github.com/MKrale/Backward_Informed_Bound}

\paragraph{Environments.}
For our experiments, we use several standard POMDP benchmarks: \tiger~\citep{DBLP:conf/aaai/CassandraKL94}, \rocksample~\citep{DBLP:conf/uai/SmithS04}, \aloha~\citep{DBLP:journals/iotj/JeonSJ22}, \tagenv~\citep{DBLP:conf/ijcai/PineauGT03}, \tigergrid~\citep{DBLP:conf/icml/LittmanCK95}, \hallwayone~\citep{DBLP:conf/icml/LittmanCK95}, \hallwaytwo~\citep{DBLP:conf/icml/LittmanCK95}, \pentagon~\citep{DBLP:conf/uai/CassandraLZ97} and \fourth~\citep{DBLP:conf/uai/CassandraLZ97}.\footnote{All environments are publically available within \emph{POMDPs.jl} or on \url{pomdps.org}.}
These environments have diverse characteristics, varying from 2 to 12545 states, from 3 to 29 actions, and from 1 to 1052 observations, as \cref{table:bounds} shows.
Additionally, we consider the \custom environment (\cref{fig:NoOpEnv}) and two new environments inspired by POMDPs in the literature. 
Firstly, we consider a $6\times 6$ grid where an agent needs to navigate from the bottom left to the top right corner. 
The observation function is as in \citet{amato2006optimal}: the agent observes in which column it is, but not in which row.
Secondly, we consider a maintenance environment called  \koutofn with the goal of keeping a number of components from breaking down~\cite{Kapur2014}.
For the latter, we add partial observability using \emph{measuring action}, which gives a negative reward but reveals the current state, and assume the agent gets no observations otherwise (similar to, e.g., \cite{DBLP:conf/ai/BellingerC0T21, DBLP:conf/nips/NamFB21, DBLP:conf/aips/KraleS023}).
\cref{ap:experiments} provides a complete description of both new environments.

\subsection{Bound Tightness and Computational Cost}%
\label{sec:expbounds}

To address \cref{q:tightness,q:overhead}, we compare the upper bounds for the initial beliefs of all environments.
In addition, we show the bounds computed by \FIB and the best lower bound found by \SARSOP within 1200s, which we consider as the closest proxy for the optimal value when evaluating the tightness of the bounds.

\textbf{\OBIB is tight but computationally intractable.}
As shown in  \cref{table:bounds}, \OBIB is always the tightest upper bound for smaller environments.
However, its computation times are significantly higher than of the other tested bounds, and for larger environment it often does not converge within the given time.

\textbf{\BIB and \EBIB are tighter than \FIB, with tractable overhead.}
\BIB and \EBIB are tighter than \FIB in all environments at the cost of longer, but mostly tractable, computation times.
The differences in the bounds are the largest for \custom, \tiger, and \gridenv, where \EBIB performs significantly better than \BIB and about on par with \OBIB.
However, for most other environment the difference between \BIB and \EBIB is minimal.
The difference between \FIB and our proposed bound is small in environments where all uncertainty is contained in the initial state, as is the case for \rocksample.

\begin{table*}[tb]
\caption{
The tightest relative value gap found by \SARSOP for different computational budgets with different bounds as initialization.
N/A denotes no lower bound has been found within the given budget.
}
\label{table:SARSOP_large}
\centering
\input{Tables/Eval_Sarsop_Large}
\end{table*}

\subsection{Improvement of \SARSOP}

Next, we investigate \cref{q:integration} by comparing the performance of \SARSOP when using \FIB, \BIB, and \EBIB as initialization.
For our evaluation, we split the environments into two groups.
For the smaller environments where \SARSOP finds an $\epsilon$-optimal policy within one hour, we consider convergence times.
For the larger environments, where \SARSOP does not converge within an hour, we instead consider the relative value gap 
\[
V_\text{gap} = \frac{\overline{V}(b_0) - \underline{V}(b_0)}{\underline{V}(b_0)}
\]
after 600s, 1200s and 3600s.
We also provide the upper- and lower bounds in \cref{ap:experiments}.
All running times include the precomputation times of the bounds.

\begin{table}[tb]
\caption{
Computation times of \SARSOP for different environments, using different heuristics as initialization.
}
\label{table:SARSOP_small}
\centering
\input{Tables/Eval_Sarsop_Small}
\end{table}

\begin{figure}[tb]
    \includegraphics[width=0.8\columnwidth]{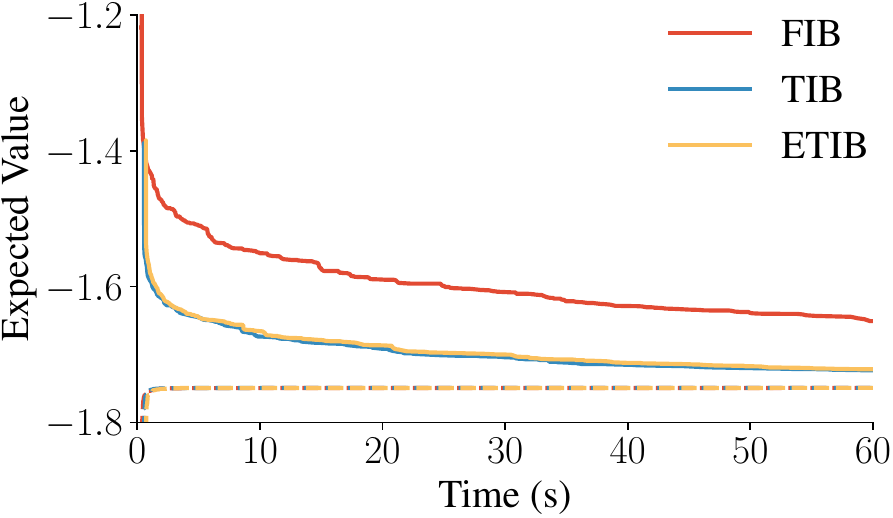}
    \caption{
    Upper- and lower bounds on the initial value of the \koutofn (2) environments as computed by \SARSOP in the first 60s, using different bounds as initialization.
    Solid lines show upper bounds, and dashed lines lower bounds.}
    \label{fig:SarsopKoutofN2}
    \Description{TODO} %
\end{figure}

\textbf{In smaller environments, using \BIB or \EBIB to initialize \SARSOP yields mixed results.}
As shown in \cref{table:SARSOP_small}, for most small environments, \SARSOP is already sufficiently fast that the initialization has little effect on computation times.
The exceptions are \koutofn (2), where using \BIB and \EBIB is significantly quicker than using \FIB, and \aloha (30), where \EBIB is significantly slower.
For \koutofn (2),  \cref{fig:SarsopKoutofN2} shows the upper- and lower bounds as computed by \SARSOP over time.
We see that when using \FIB, an initial upper bound is computed slightly quicker, but the convergence speed is worse than when using \BIB or \EBIB.

\textbf{In larger environments, using \BIB improves the bounds computed by \SARSOP.}
\Cref{table:SARSOP_large} shows the tightest relative value gap found with different computational budgets for our larger environments.
We see that using \BIB typically improves the performance of \SARSOP given a sufficiently large computational budget.
However, for environments where \BIB is computationally expensive (such as \rocksample (7,8) and \fourth), we find that using \FIB (initially) yields better results.
In contrast, using \EBIB yields better results for \gridenv, but otherwise performs similar or worse then \BIB due to its higher computational cost.

\begin{figure*}[tbh]
\center{
        \includegraphics[width=0.32\textwidth]{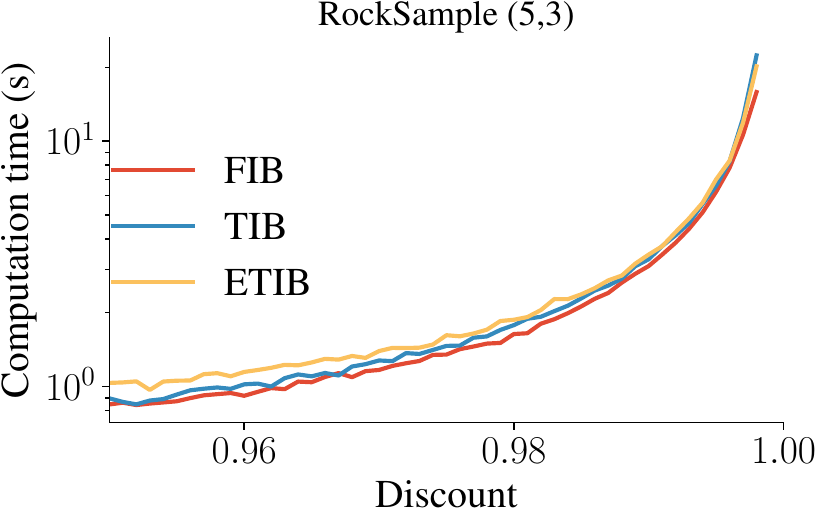}
        \includegraphics[width=0.32\textwidth]{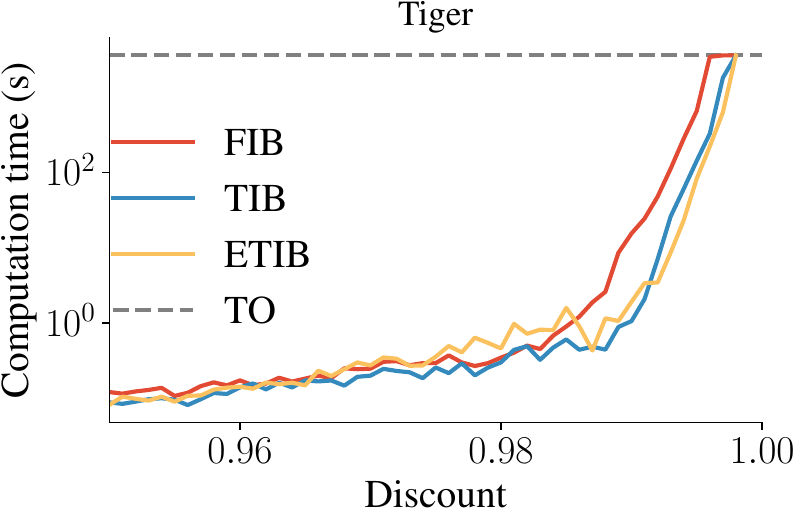}  
        \includegraphics[width=0.32\textwidth]{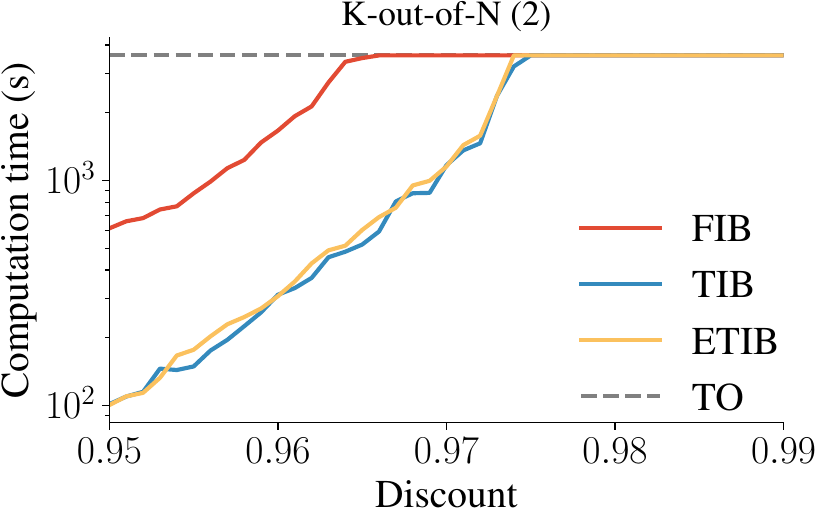}
        }
    \caption{Computation times of \SARSOP against the discount factor, using different bounds as initialization.}
    \label{fig:DiscountDependency}
    \Description{}
\end{figure*}

\subsection{Discount dependency}

Lastly, we investigate \cref{q:discount} by testing how the discount factor affects the computation times of \SARSOP, given different initialization bounds.
Due to the computational cost, we only test on smaller environments, but we expect this behavior to translate to larger environments as well.

\textbf{\SARSOP{} profits more from tighter initial bounds for high discount factors.}
As shown in \cref{fig:DiscountDependency}, the effect of different initialization is minimal for goal-oriented environments, such as  \rocksample~(5).
However, for non-goal-oriented problems (such as \tiger and \koutofn), we find that the absolute speedup of using \BIB and \EBIB increases with the discount factor.

\section{Related Work}

Besides QMDP~\cite{DBLP:conf/icml/LittmanCK95}, FIB~\cite{DBLP:journals/jair/Hauskrecht00} and point-based methods~\cite{DBLP:journals/ior/Lovejoy91,DBLP:conf/ijcai/ZhouH01,DBLP:conf/icml/Bonet02,DBLP:conf/ijcai/PineauGT03,DBLP:conf/uai/SmithS04,DBLP:conf/uai/SmithS05} which we introduced in \cref{sec:background}, we mention a number of other methods used for computing upper bounds.
Firstly, a number of works consider simplifying the set of reachable beliefs by discretizing the belief space in a similar style as point-based solvers~\cite{DBLP:conf/tacas/BorkKQ22, DBLP:journals/corr/abs-2104-07276, DBLP:conf/iros/WrayZ17, DBLP:journals/rts/Norman0Z17,DBLP:conf/atva/BorkJKQ20}.
Next, \citet{DBLP:conf/aaai/YoonFGK08} introduce `hindsight optimization', which uses deterministic planning in a number of sampled `situations' to approximate the value of a (PO)MDP.
\citet{DBLP:journals/tac/HaughL20} use `information relaxation' in a similar way.
\citet{DBLP:conf/nips/BarenboimI23} consider only a subset of possible outcomes of the transition- and observation function to compute upper bounds in online solvers.
However, all these methods are typically less tight (though computationally cheaper) than our proposed bounds.
Lastly, some bounds are based on the properties of a particular type of POMDP.
For example, \citet{DBLP:journals/cor/Sinuany-SternDB97} consider POMDPs that model maintenance, while \citet{DBLP:conf/aips/KraleS023} consider POMDPs where agents have explicit measuring actions.

Our empirical analysis focuses on \SARSOP~\cite{DBLP:conf/rss/KurniawatiHL08}, but we mention a few related state-of-the-art POMDP solvers.
Firstly, POMCP \cite{DBLP:conf/nips/SilverV10}, and AdaOPS \cite{DBLP:conf/nips/WuYZYLLH21} are both variants of Monte Carlo tree search (MCTS) adapted for POMDPs.
DESPOT \cite{DBLP:journals/jair/YeSHL17} is also based on tree search but uses a method based on hindsight optimization to increase tractability.
Lastly, many methods make use of (deep) reinforcement learning to find approximate solutions to POMDPs \cite{DBLP:conf/aaaifs/HausknechtS15,DBLP:conf/nips/LeeNAL20,DBLP:conf/iclr/HanDT20}.
However, all these methods focus on large (continuous) state-, action- and observation spaces where our bounds are computationally intractable.

\emph{Deterministic Delay MDPs} (DDMDPs)~\cite{DBLP:conf/sigmetrics/AltmanN92,DBLP:journals/tac/KatsikopoulosE03,DBLP:journals/aamas/WalshNLL09} are MDPs where the agent can fully observe its state with some constant delay, which is conceptually similar to FIB and TIB. 
Finding exact solutions to DDMDPs is NP-hard~\cite{DBLP:journals/aamas/WalshNLL09}, but efficient approximate solvers exist~\cite{DBLP:journals/prl/AgarwalA21,DBLP:conf/iclr/DermanDM21}.
However, FIB and TIB  take into account (partial) observations occurring before the state is fully revealed, while such observations do not exist in DDMDPs.
This means that in POMDPs with no observations, FIB and TIB correspond to the solutions of DDMDPs with delays 1 and 2, respectively.
However, solutions for DDMDPs are not sound upper bounds for POMDPs in general, so we do not compare our method with DDMDP solvers.

Lastly, we mention a number of other works related to $\epsilon$-optimal POMDP solving.
\citet{DBLP:conf/aaai/WalravenS17} and \citet{DBLP:conf/uai/HansenB20} propose methods to speed up the incremental pruning of $\alpha$-vectors, which constitutes a considerable amount of the computation time of \SARSOP. 
Relatedly, \citet{DBLP:conf/aaai/DujardinDC17} propose a method that uses less $\alpha$-vectors instead.
\citet{DBLP:conf/aaai/WangPBS06} proposes to use quadratic functions instead of piecewise-linear functions to represent the upper bound.

\section{Conclusion}

To improve the performance of $\epsilon$-optimal solvers, we introduced three novel bounds for POMDPs (\BIB, \OBIB, and \EBIB).
We prove these bounds are tighter than the commonly used \FIB, and show they can be computed using value iteration.
Empirically, both \BIB and \EBIB are computationally tractable on a large range of benchmarks.
Moreover, using these bounds to initialize state-of-the-art solver \SARSOP improves its performance.

Future work may focus on increasing the tractability of our bounds.
For example, instead of computing bounds for all beliefs $b \in \Bsao$, it may be quicker to use \FIB for those that have a low probability of being reached.
Alternatively, more research could be done on different heuristic choices for weights, particularly choices that do not require solving LPs.
We consider and test one such choice in \cref{ap:experiments} with limited success, but using different (combinations of) heuristic(s) could potentially get tight bounds at lower computational costs than \EBIB.
Lastly, future work could consider how our bounds can be applied to other settings, such as finite- or indefinite horizon problems.

\balance

\bibliographystyle{ACM-Reference-Format}
\bibliography{references}

\newpage
\onecolumn
\appendix

\section{Technical Details Experiments}
\label{ap:experiments}

In this section, we provide additional details about our experimental evaluation that did not fit in the main paper.
All code can be found in the supplementary materials and will be made publicly available after publication.
All experiments are performed on a 3.7GHz Intel Core i9 CPU running Ubuntu 22.04.1 LTS. 

\subsection{Implementation details}

In this section, we give a high-level overview of how we implemented the precomputation of our bounds.

\subsubsection{\FIB and QMDP}
We initially used the \emph{FIB.jl} framework for computing \FIB, but found that it did not perform well on larger environments.
Thus, we used a custom implementation to precompute \FIB using value iteration.
To speed up convergence, we additionally wrote a custom implementation for QMDP (using value iteration) that we use as an initialization for \FIB. 
To guarantee we obtain sound upper bounds, we initialize QMDP using $\max\limits_{(s,a) \in \states \times \actions} R(s,a) \frac{1}{1-\gamma}$ for all state-action pairs.

\subsubsection{\BIB, \EBIB and \OBIB}
The precomputation of \BIB, \EBIB, and \OBIB essentially follows \cref{alg:BIB}.
To improve computational speeds, we cache transition probabilities and reachable beliefs for all beliefs in $\Bsao$.
For solving LPs, we use the \emph{JuMP.jl} framework~\cite{Lubin2023} and the Clp optimizer~\cite{john_forrest_2024_13347196}, which we found to perform best in practice.
We initialize our values using \FIB, which guarantees soundness.

\subsubsection{Implementation \SARSOP}
Our experiments use the native Julia \SARSOP solver (\url{https://github.com/JuliaPOMDP/NativeSARSOP.jl}), which we alter in two ways:

\paragraph{Initialisation with different bounds:}
Firstly, to be able to use our bounds, we altered \SARSOP such that it can use any solver to compute an initial upper bound.
More precisely, the Julia implementation of \SARSOP uses a vector \emph{cornervalues} to represent the values of beliefs $\Bs$, which are initialized using \FIB by default.
We alter the code so that it can be initialized using \emph{any} solver.
We additionally add two vectors \emph{B\_heuristic} and \emph{V\_heuristic}, which get treated as part of the point set when computing upper bounds.
For our experiments, however, we initialize these with empty vectors, meaning they have no impact.

\paragraph{Relative precision:}
Secondly, by default, \SARSOP uses an \emph{absolute} target precision $\epsilon$ as a stopping condition and a planning precision $\epsilon_p$ to determine the sampling depth.
To allow for better comparison between benchmarks, we alter \SARSOP such that the precision is \emph{relative}, i.e. the algorithm terminates when $\frac{\bar{V}(b_0) - \underline{V}(b_0)}{\underline{V}(b_0)} \leq \epsilon$.
Prior works have found that picking $\epsilon_p > \epsilon$ often gives better performance, but then there is no guarantee that an $\epsilon$-optimal policy can ever be found.
As a compromise between practical and theoretical performance, we let $\epsilon_p$ linearly decrease with the number of iterations, as $\epsilon_p = \max \{ \epsilon,  \frac{i+1}{N+1}\}$, with $i$ the iteration number and $N=1000$ a hyperparameter.
In practice, we found that this yields good performance on both large and small benchmarks.

\subsection{Environment Descriptions}

Next, we give a short overview of our novel benchmarks environments.
Both are implemented using the \emph{POMDPs.jl} framework and can thus be used by others.

\paragraph{Grid}
We consider a $6 \times 6$ Grid POMDP. 
This POMDP has 36 states, each representing that the agent is in a particular cell in the grid. 
The possible actions are staying in-place and moving in one of the four directions: up, down, left, right. 
Staying in place is always successful. 
However, moving is only successful with probability 0.6, with there also being a probability of 0.1 for moving in each of the other three directions or staying in-place. 
If moving in a particular direction would cause the agent to go outside the grid, the probability for that direction is instead added to the probability for staying in-place.
The goal is to navigate to the top right corner. 
Accordingly, the reward is 1 when the \emph{new} state is the top right corner, and 0 otherwise (using the transition function, this can be translated to expected rewards given the old state). 
The initial belief assigns probability 1 to the bottom left corner. 
There are six observations, one for each column. 
Regardless of the action taken, the agent is always informed in which column it is, but not in which row it is. 

\paragraph{K-out-of-N}
We consider a POMDP that consists of $N$ components, which can all be in one of $s_\text{max} = 4$ degradation states, where degradation state $s_\text{max}$ represents a broken component.
For each component and at each time step, the agent can choose either to do nothing, inspect the component, or repair it.
When repairing, the component returns to a non-degraded state.
Otherwise, each component has a chance to degrade (i.e., its degradation state increases by 1), with a probability that depends on the number of other components that are currently broken.
More precisely, with probability $0.2$ if no components are broken, probability $0.5$ if one component is broken, and $0.9$ otherwise.
The agent cannot observe the degradation state of each component unless they take an inspection action.
However, this yields a negative reward of $0.05$.
Moreover, the agent gains a reward of $-0.5$ for every broken component and a negative reward of $0.25$ when repairing.
\koutofn (x) denotes a K-out-of-N problem with x components.
We note that using a larger number of components increases both the state space as well as the action- and observation spaces.

\subsection{Other heuristic choices: CTIB}

\begin{table}[tb]
\caption{
Upper bounds and computation times of CTIB on a number of POMDP benchmarks, as compared to \BIB and \EBIB.
If a bound has not converged within 1200s, we report the last computed bound and denote computation time as TO.
The tightest bounds are bolded.
}
\label{table:Bounds_CTIB}
\centering
\input{Tables/Eval_Bounds_CTIB}
\end{table}

In response to one of our reviewers, we implement and test an additional heuristic which we refer to as \emph{Closeness-based Tighter Informed Bound} (CTIB).
For CTIB, for any belief $b$ we first find the belief $b' \in \Bsao$ that has the highest \emph{minimum ratio} \cite{Kochenderfer2022}:
\begin{equation}
    b_\text{closest}(b) = \argmax_{b' \in \Bsao} \min_{s \in \states} \frac{b(s)}{b'(s)}.
\end{equation}
Given $b_\text{closest}$, we construct a weight function using only this belief and beliefs in $\Bs$, as follows:
\begin{equation}
    \went'_{b,a,o}(b') = \begin{cases}
        \min_{s \in \states} \frac{b(s)}{b'(s)} & \text{ if } b'=b_c(b) \\
        \sum_s \max \big(0, b(s) - b_c(s)\big) & \text{ if } b \in \Bs\\
        0 & \text{ otherwise}
    \end{cases} 
\end{equation}
This weight function performs well if the belief with the highest minimum ratio also has high state uncertainty, which is often the case.
Moreover, since finding $b_\text{closest}$ does not require solving an LP, it can be significantly quicker than ETIB for some benchmarks.

We repeat the experiments to answer questions \cref{q:tightness} and \cref{q:tightness} (\cref{sec:expbounds}) for CTIB.
The results are shown in \cref{table:Bounds_CTIB}.
As expected, CTIB is computationally cheaper than \EBIB, particularly for larger environments.
However, its tightness is less consistent: CTIB never finds tighter bounds than \EBIB and even performs worse than \BIB on some benchmarks, such as on the two \koutofn benchmarks.
Thus, as mentioned in the conclusion, we believe more research is required into cheaper heuristics before they are a viable alternative to \BIB or \EBIB.

\subsection{Further experimental findings}

All experimental results are contained in a Jupyter notebook, which is included in the supplementary material.
Here, we include a variant of \cref{table:SARSOP_large}, which shows both upper- and lower bounds for each computational budget (instead of only the relative value gap).
Our bounds mostly impact the upper bounds found by \SARSOP.
For the lower bounds, we find that our bounds have a negligible effect.

\begin{table*}[tb]
\footnotesize
\begin{minipage}{.99\textwidth}
\centering
\resizebox{\textwidth}{!}{
\input{Tables/Sarsop_Large_full}
}
\end{minipage}
\caption{
Upper- and lower bounds found by \SARSOP for different computational budgets with different bounds as initialization.
N/A denotes no lower bound has been found within the given budget.
}
\label{table:SARSOP_large_full}
\end{table*}

\newpage

\section{Proofs}
\label{ap:proofs}

In this section, we provide detailed proofs of all theorems stated in the main text. 
We first provide some auxiliary lemmas, and then we proceed by proving the main results, ordered by bound. 
We also provide results for \FIB, since our exposition differs slightly from the exposition in \citet{DBLP:journals/jair/Hauskrecht00}.  
Finally, we provide some general properties of the value iteration algorithm that apply to each of our bounds. 

\subsection{Auxiliary Lemmas}

\subsubsection{Weight Functions} 
We start by showing a simple property of weight functions. 
\begin{lemma} \label{lem:weightfun}
Let $\mathcal{B}$ be a point set.  If $w \in \weightset_{\mathcal{B}, b}$ is a weight function, then $\sum_{b' \in \mathcal{B}}   w(b') = 1$.
\end{lemma}
\begin{proof} By \cref{def:weightf}, a weight function $w \in \weightset_{\mathcal{B}, b}$ satisfies $\sum_{b' \in \mathcal{B}}   w(b') b'(s) = b(s)$ for all $s \in \states$. Summing these equalities and interchanging the order of summation yields
\[
\sum_{b' \in \mathcal{B}}   w(b') = \sum_{b' \in \mathcal{B}}   w(b') \left(\sum_{s \in \states} b'(s) \right)  =  \sum_{b' \in \mathcal{B}}  \sum_{s \in \states}  w(b') b'(s) = \sum_{s \in \states} \sum_{b' \in \mathcal{B}}   w(b') b'(s) = \sum_{s \in \states} b(s) = 1,
\]
where we use that $b$ and each $b'$ are probability distributions and hence sum to 1 over their domain $\states$.
\end{proof}

\subsubsection{Contraction Mappings} 
Next, we prove a property that we will use when showing that our Bellman operators are contraction mappings.

\begin{lemma} \label{lem:contraction}
Let $X$ be a set and let $f, g \colon X \rightarrow \mathbb{R}$ be functions. Then \[\left|\sup_{x\in X} f(x) - \sup_{x\in X} g(x)\right| \leq \sup_{x\in X} |f(x) - g(x)| \qquad \text{and} \qquad \left|\inf_{x\in X} f(x) - \inf_{x\in X} g(x)\right| \leq \sup_{x\in X} |f(x) - g(x)|.\]
\end{lemma}

\begin{proof} We have 
\[
\sup_{x\in X} f(x) - \sup_{x\in X} g(x) = \sup_{x\in X} \left[f(x) - \sup_{x'\in X} g(x')\right] \leq  \sup_{x\in X} \left[f(x) - g(x)\right] \leq \sup_{x\in X} |f(x) - g(x)|,
\]
where for the middle inequality we use that $\sup\limits_{x'\in X} g(x') \geq g(x)$ for any $x \in X$. By exchanging the role of $f$ and $g$, we similarly get 
\[
\sup_{x\in X} g(x) - \sup_{x\in X} f(x) \leq \sup_{x\in X} |g(x) - f(x)| = \sup_{x\in X} |f(x) - g(x)|,
\]
which altogether implies the first inequality $\displaystyle \left|\sup_{x\in X} f(x) - \sup_{x\in X} g(x)\right| \leq \sup_{x\in X} |f(x) - g(x)|$. 

For the second inequality, we apply the first inequality to the functions $-g$ and $-f$, which yields 
\[ \left|\inf_{x\in X} f(x) - \inf_{x\in X} g(x)\right|  = \left|\sup_{x\in X} (-g)(x) - \sup_{x\in X} (-f)(x)\right| \leq \sup_{x\in X} |(-g)(x) - (-f)(x)| = \sup_{x\in X} |f(x) - g(x)|. \qedhere \]
\end{proof}

\noindent In practice, when we apply Lemma \ref{lem:contraction}, we know that the suprema and infima on the left hand side are actually attained (and are hence written as maxima and minima), although on the right side we will write suprema when $X$ is an infinite set. 

\subsubsection{Uniqueness of the Fixed Point.} Next, we explain why showing that our Bellman operators are contraction mappings is sufficient to show that they have a unique fixed point. 
For this, we recall Banach's fixed point theorem \cite{banach1922operations}:

\begin{theorem}[Banach's fixed point theorem]\label{thm:banach}  Let $(X, d)$ be a non-empty complete metric space and let $T\colon X \rightarrow X$ be a contraction mapping. Then $T$ has a unique fixed point $x^*$. Moreover, we can find $x^*$ in the limit by starting from an arbitrary element $x \in X$ and repeatedly applying $T$: if we define the sequence $(x_n)_{n\geq 0}$ by $x_0 = x$ and $x_{n+1} = T(x_n)$, then $x^* = \lim\limits_{n\to\infty}x_n$. 
\end{theorem}

Recall that $\setQ$ is  the set of functions $Q\colon \distr{\states} \times \actions \rightarrow \mathbb{R}$. Using Banach's fixed point theorem, we prove the following:

\begin{lemma} \label{lem:banach}
Let $\mathcal{B}$ be a finite point set, let $\setQ_{\mathcal{B}}$ be the set of functions $Q\colon \mathcal{B} \times \actions \rightarrow \mathbb{R}$, and let $H  \colon \setQ \rightarrow \setQ$ be a contraction mapping. Equip $\setQ$ and $\setQ_{\mathcal{B}}$ with the $\infty$-norm.
Assume that we can compute $HQ(b,a)$ for any $(b,a) \in \distr{\states} \times \actions$ using only the $Q$-values on $\mathcal{B}  \times \actions$. Then we can also see $H$ as an operator mapping $\setQ_{\mathcal{B}}$ to $\setQ$, which induces an operator $H|_{\mathcal{B}} \colon \setQ_{\mathcal{B}} \rightarrow \setQ_{\mathcal{B}}$ by restricting the domain $\distr{\states} \times \actions \rightarrow \mathbb{R}$ of the each function $HQ$ to $\mathcal{B} \times \actions \rightarrow \mathbb{R}$. 
Then $H|_{\mathcal{B}} \colon \setQ_{\mathcal{B}} \rightarrow \setQ_{\mathcal{B}}$ has a unique fixed point $Q^*|_{\mathcal{B}}\colon \mathcal{B} \times \actions \rightarrow \mathbb{R}$. Moreover, if we can compute the $Q$-value for any belief $b \in \distr{\states}$ using the $Q$-values on $\mathcal{B}$, this induces a unique fixed point $Q^*\colon  \distr{\states} \times \actions \rightarrow \mathbb{R}$ of $H$. 
\end{lemma}
\begin{proof}
Since $\mathcal{B}$ is finite and $\actions$ is finite, the set  $\setQ_{\mathcal{B}}$ can be seen as a finite-dimensional normed vector space (isomorphic to $\mathbb{R}^{|\mathcal{B} \times \actions|}$). Since the metric space corresponding to any finite-dimensional normed vector space is complete, Banach's fixed point theorem (\cref{thm:banach}) now implies that $H|_{\mathcal{B}}$ has a unique fixed point $Q^*|_{\mathcal{B}} \colon \mathcal{B} \times \actions \rightarrow \mathbb{R}$. Since we can compute the $Q$-value for any belief $b \in \distr{\states}$ using the $Q$-values on $\mathcal{B}$, we can use this to compute a function $Q^*\colon  \distr{\states} \times \actions \rightarrow \mathbb{R}$. Since $H$ fixes the restriction of $Q^*$ to $\mathcal{B} \times \actions$, and these values determine all other values, it follows that $H$ fixes $Q^*$.
\end{proof}

\noindent Note that we go via the finite-dimensional space  $\setQ_{\mathcal{B}}$ here, since infinite-dimensional spaces are not necessarily complete. 

\subsubsection{Inequalities.} Finally, we prove a lemma that we can use to show inequalities between different $Q$-functions. For this, we will use that all Bellman operators $H$ we consider in this paper ($H_\FIB$, $H_\BIB$, $H_{\OBIB}$, $H_\EBIB$ and $H_\POMDP$) are \emph{monotone}: If $Q \leq Q'$ pointwise (i.e.\ $Q(b,a) \leq Q'(b,a)$ for all $(b,a) \in \distr{\states} \times \actions$), then also $HQ \leq HQ'$ pointwise. With this, we can show the following:

\begin{lemma} \label{lem:ineq}
Let $\mathcal{B}$ be a finite point set, let $\setQ_{\mathcal{B}}$ be the set of functions $Q\colon \mathcal{B} \times \actions \rightarrow \mathbb{R}$, and let $H_\mathrm{X}, H_\mathrm{Y}  \colon \setQ \rightarrow \setQ$ be monotone contraction mappings. Assume that $ H_\mathrm{Y}$ has the property that we can compute $H_\mathrm{Y}Q(b,a)$ for any $(b,a) \in \distr{\states} \times \actions$ using only the $Q$-values on $\mathcal{B}  \times \actions$. Let $Q_\mathrm{X}$ and $ Q_\mathrm{Y}$ be the fixed points of $H_\mathrm{X}$ and $ H_\mathrm{Y}$. Assume that $H_\mathrm{Y} Q_\mathrm{X} \leq Q_\mathrm{X}$. Then $Q_\mathrm{Y} \leq Q_\mathrm{X}$. Similarly, if $H_\mathrm{Y} Q_\mathrm{X} \geq Q_\mathrm{X}$, then $Q_\mathrm{Y} \geq Q_\mathrm{X}$.
\end{lemma}

\begin{proof}
Since the proofs  of the `$\leq$'-statement and the `$\geq$'-statement are similar, we only prove the  `$\leq$'-statement. Define a sequence $(Q_n)_{n\geq0}$ by $Q_0 = Q_\mathrm{X}$ and $Q_{n+1} = H_\mathrm{Y} Q_n$. We prove by induction on $n \geq 0$ that $Q_n \leq Q_\mathrm{X}$. The base case states $Q_\mathrm{X} \leq Q_\mathrm{X}$, so we proceed with the induction step. Assume that $Q_n \leq Q_\mathrm{X}$. Applying the monotone operator $H_\mathrm{Y}$ to both sides then yields $Q_{n+1} = H_\mathrm{Y} Q_n \leq H_\mathrm{Y} Q_\mathrm{X}$. Moreover, we are given that $H_\mathrm{Y} Q_\mathrm{X} \leq Q_\mathrm{X}$, so we conclude that $Q_{n+1} \leq Q_\mathrm{X}$. This completes the induction, so we conclude that $Q_n \leq Q_\mathrm{X}$ for all $n \geq 0$. By Banach's fixed point theorem (applied in a similar way as in \cref{lem:banach}), it follows that $\lim\limits_{n \to\infty} Q_n = Q_\mathrm{Y}$ on $\mathcal{B}\times \actions$. By taking the limit $n \to \infty$ in the inequality  $Q_n \leq Q_\mathrm{X}$, it hence follows that $Q_\mathrm{Y}  \leq Q_\mathrm{X}$  on $\mathcal{B}\times \actions$. But since the values on $\mathcal{B}\times \actions$ are the only ones that are relevant for computing $H_\mathrm{Y} Q$ and since $H_\mathrm{Y} Q$ is monotone, it follows that $Q_\mathrm{Y}  = H_\mathrm{Y}  Q_\mathrm{Y}  \leq H_\mathrm{Y} Q_\mathrm{X} \leq Q_\mathrm{X}$.  
\end{proof}

\subsection{Fast Informed Bound}
\label{ap:fib_proofs}

We start by providing proofs for some properties of the Fast Informed Bound, since our exposition is slightly different from the exposition in \citet{DBLP:journals/jair/Hauskrecht00} and since we will actually use these properties directly in our proof. 

\begin{theorem} 
$H_\FIB$ is a contraction operator w.r.t.\ the  $\infty$-norm with Lipschitz constant $\gamma < 1$, and $H_\FIB$ has a unique fixed point $\qfib$. 
\end{theorem}

\begin{proof}
Let $Q, Q' \colon \distr{\states} \times \actions \rightarrow \mathbb{R}$ be any $Q$-functions. Fix $b \in \distr{\states}$ and $a \in \actions$. Then we have
\begin{align*}
    | H_\FIB Q(b,a) - H_\FIB Q'(b,a) | &= \left| \gamma\sum_{\obs \in \observations}  \max\limits_{\act' \in \actions} \sum_{s' \in \states}  \left[ \Pr(\obs, s' \midd b,a) Q(\bel_{s'},\act') \right] - \gamma\sum_{\obs \in \observations}  \max\limits_{\act' \in \actions} \sum_{s' \in \states}  \left[ \Pr(\obs, s' \midd b,a) Q'(\bel_{s'},\act') \right] \right| \\
    &\leq  \gamma \sum_{\obs \in \observations}  \left| \max\limits_{\act' \in \actions} \sum_{s' \in \states}  \left[ \Pr(\obs, s' \midd b,a) Q(\bel_{s'},\act') \right] -  \max\limits_{\act' \in \actions} \sum_{s' \in \states}  \left[ \Pr(\obs, s' \midd b,a) Q'(\bel_{s'},\act') \right]\right|  \\
    &\leq  \gamma \sum_{\obs \in \observations} \max\limits_{\act' \in \actions}  \left| \sum_{s' \in \states}  \left[ \Pr(\obs, s' \midd b,a) Q(\bel_{s'},\act') \right] - \sum_{s' \in \states}  \left[ \Pr(\obs, s' \midd b,a) Q'(\bel_{s'},\act') \right] \right|  \\
    &\leq  \gamma \sum_{\obs \in \observations} \max\limits_{\act' \in \actions}  \sum_{s' \in \states} \Pr(\obs, s' \midd b, \act)  \left|  Q(\bel_{s'},\act') - Q'(\bel_{s'},\act') \right|  \\
    &\leq  \gamma \sum_{\obs \in \observations} \max\limits_{\act' \in \actions}  \sum_{s' \in \states} \Pr(\obs, s' \midd b, \act) \left\|   Q- Q' \right\|_{\infty} \\
    &= \gamma \left\|   Q- Q' \right\|_{\infty} \sum_{\obs \in \observations}  \sum_{s' \in \states} \Pr(\obs, s' \midd b, \act) = \gamma \left\|   Q- Q' \right\|_{\infty},
\end{align*}
where the inequalities follow from the triangle inequality, \cref{lem:contraction}, the triangle inequality and the definition of the $\infty$-norm. 

By the definition of the $\infty$-norm, it follows that
\[
\left\|   H_\FIB Q - H_\FIB Q' \right\|_{\infty} \leq \max_{(b,a) \in \distr{\states} \times \actions} | H_\FIB Q(b,a) - H_\FIB Q'(b,a) | \leq \gamma \left\|   Q- Q' \right\|_{\infty},
\]
which shows that  $H_\FIB$ is a contraction operator with Lipschitz constant $\gamma$. Since $H_\FIB$ can be computed using only the finite set of unit beliefs, \cref{lem:banach} now implies that $H_\FIB$ has a unique fixed point $\qfib$.
\end{proof}

\begin{theorem}  \label{lem:fibconvex}
Let $\act \in \actions$ be given. Then  $b \mapsto \qfib(b,a)$ is convex.
\end{theorem}

\begin{proof}
Let $Q$ be any $Q$-function. We first show that for given $\act \in \actions$, the function $b \mapsto H_\FIB Q(b, \act)$ is convex. 
The function $b \mapsto   \Pr(\obs, s' \midd b, \act) = \sum_{s \in \states} b(s) \Pr(\obs, s' \midd s, \act)$ is linear in $b$ (since $\Pr(\obs, s' \midd s, \act)$ does not depend on $b$). Hence, also \[b \mapsto \sum_{s' \in \states}  \left[ \Pr(\obs, s' \midd b,a) Q(\bel_{s'},\act') \right]\] is linear in $b$ (since the values of $Q(\bel_{s'},\act')$ do not depend on $b$), and hence convex in $b$. 
Since the pointwise maximum of convex functions is convex and the sum of convex functions is convex, it follows that 
\[
b \mapsto  \sum_{\obs \in \observations}  \max\limits_{\act' \in \actions} \sum_{s' \in \states}  \left[ \Pr(\obs, s' \midd b,a) Q(\bel_{s'},\act') \right]
\]
is convex in $b$. Finally, since $b \mapsto R(b,a)$ is linear (and hence convex) in $b$, we conclude that $b \mapsto H_\FIB Q(b, \act)$ is convex by \cref{eq:FIB}. Since $\qfib$ satisfies $\qfib = H_\FIB Q$, this in particular implies that $b \mapsto \qfib(b,a)$ is convex for any given $\act \in \actions$.
\end{proof}

\subsection{Tighter Informed Bound}
\label{ap:tib_proofs}

Next, we turn to \BIB. We first show that the fixed point $\qbib$ exists and is unique.

\begin{theorem}  \label{thm:qbib_exists}
$H_\BIB$ is a contraction operator w.r.t.\ the  $\infty$-norm with Lipschitz constant $\gamma < 1$, and $H_\BIB$ has a unique fixed point $\qbib$. 
\end{theorem}

\begin{proof}
Let $Q, Q' \colon \distr{\states} \times \actions \rightarrow \mathbb{R}$ be any $Q$-functions. Fix $b \in \distr{\states}$ and $a \in \actions$. Then we have
\begin{align*}
    | H_\BIB Q(b,a) - H_\BIB Q'(b,a) | &= \left| \gamma\sum_{\obs \in \observations}  \max\limits_{\act' \in \actions} \sum_{s \in \states} \left[ b(s) \Pr(\obs \midd s, \act)  Q(\bel_{s,\act,\obs}, \act') \right] - \gamma\sum_{\obs \in \observations}  \max\limits_{\act' \in \actions} \sum_{s \in \states} \left[ b(s) \Pr(\obs \midd s, \act)  Q'(\bel_{s,\act,\obs}, \act') \right]\right| \\
    &\leq  \gamma \sum_{\obs \in \observations}  \left| \max\limits_{\act' \in \actions} \sum_{s \in \states} \left[ b(s) \Pr(\obs \midd s, \act)  Q(\bel_{s,\act,\obs}, \act') \right] -  \max\limits_{\act' \in \actions} \sum_{s \in \states} \left[ b(s) \Pr(\obs \midd s, \act)  Q'(\bel_{s,\act,\obs}, \act') \right]\right|  \\
    &\leq  \gamma \sum_{\obs \in \observations} \max\limits_{\act' \in \actions}  \left| \sum_{s \in \states} \left[ b(s) \Pr(\obs \midd s, \act)  Q(\bel_{s,\act,\obs}, \act') \right] -  \sum_{s \in \states} \left[ b(s) \Pr(\obs \midd s, \act)  Q'(\bel_{s,\act,\obs}, \act') \right]\right|  \\
    &\leq  \gamma \sum_{\obs \in \observations} \max\limits_{\act' \in \actions}  \sum_{s \in \states} b(s) \Pr(\obs \midd s, \act)  \left|   Q(\bel_{s,\act,\obs}, \act') - Q'(\bel_{s,\act,\obs}, \act') \right|  \\
    &\leq  \gamma \sum_{\obs \in \observations} \max\limits_{\act' \in \actions}  \sum_{s \in \states} b(s) \Pr(\obs \midd s, \act)  \left\|   Q- Q' \right\|_{\infty} \\
    &= \gamma \left\|   Q- Q' \right\|_{\infty} \sum_{\obs \in \observations}  \sum_{s \in \states} b(s) \Pr(\obs \midd s, \act) = \gamma \left\|   Q- Q' \right\|_{\infty},
\end{align*}
where the inequalities follow from the triangle inequality, \cref{lem:contraction}, the triangle inequality and the definition of the $\infty$-norm. 

By the definition of the $\infty$-norm, it follows that
\[
\left\|   H_\BIB Q - H_\BIB Q' \right\|_{\infty} \leq \max_{(b,a) \in \distr{\states} \times \actions} | H_\BIB Q(b,a) - H_\BIB Q'(b,a) | \leq \gamma \left\|   Q- Q' \right\|_{\infty},
\]
which shows that  $H_\BIB$ is a contraction operator with Lipschitz constant $\gamma$. Since $H_\BIB$ can be computed using only the finite set of one-step beliefs, \cref{lem:banach} now implies that $H_\BIB$ has a unique fixed point $\qbib$.
\end{proof}

\noindent Next, we show that $\qbib$ is convex in $b$. 

\begin{theorem}  \label{lem:bibconvex}
Let $\act \in \actions$ be given. Then  $b \mapsto \qbib(b,a)$ is convex.
\end{theorem}

\begin{proof}
Let $Q$ be any $Q$-function. We first show that for given $\act \in \actions$, the function $b \mapsto H_\BIB Q(b, \act)$ is convex. 
The function $b \mapsto  \sum_{s \in \states} \left[ b(s) \Pr(\obs \midd s, \act)  Q(\bel_{s,\act,\obs}, \act') \right]$ is linear in $b$ (since $ \Pr(\obs \midd s, \act)  Q(\bel_{s,\act,\obs}, \act')$ does not depend on $b$), and hence convex in $b$.
Since the pointwise maximum of convex functions is convex and the sum of convex functions is convex, it follows that 
\[
b \mapsto  \sum_{\obs \in \observations}  \max\limits_{\act' \in \actions} \sum_{s \in \states} \left[ b(s) \Pr(\obs \midd s, \act)  Q(\bel_{s,\act,\obs}, \act') \right]
\]
is convex in $b$. Finally, since $b \mapsto R(b,a)$ is linear (and hence convex) in $b$, we conclude that $b \mapsto H_\BIB Q(b, \act)$ is convex by \cref{eq:BIB}. Since $\qbib$ satisfies $\qbib = H_\BIB Q$, this in particular implies that $b \mapsto \qbib(b,a)$ is convex for any given $\act \in \actions$.
\end{proof}

\noindent The next theorem shows that \BIB is a sound upper bound.

\begin{theorem}
For any belief $b \in \distr{\states}$ and any $\act \in \actions$, it holds that $\qbib(b, \act) \geq \qpomdp(b,\act)$.
\end{theorem}

\begin{proof}
By \cref{thm:obib_sound} and \cref{thm:obib_bib}, we have that $\qbib(b, \act) \geq \qobib(b,\act) \geq \qpomdp(b,\act)$.
\end{proof}

\noindent Finally, we show that TIB is indeed tighter than FIB.

\begin{theorem} 
For any belief $b \in \distr{\states}$ and any $\act \in \actions$, it holds that $\qbib(b, \act) \leq \qfib(b,\act)$.
\end{theorem}
\begin{proof}
Let $Q$ be any convex $Q$-function. We first show that $H_{\BIB}Q \leq H_{\FIB}Q$, i.e. that $H_{\BIB}Q(b, \act) \leq H_{\FIB}Q(b,\act)$ for all $b \in \distr{\states}$ and $\act \in \actions$.
 We write the belief $\bel_{s,\act,\obs}$ as a convex combination of state beliefs:
$
\bel_{s, \act, \obs} = \sum_{s' \in \states} \left[\frac{\Pr(s', \obs \midd s, \act)}{\Pr(\obs \midd s, \act)} \bel_{s'}\right].$
Hence, the convexity of $Q$ implies
\[
Q(\bel_{s, \act, \obs}, \act') \leq \sum_{s' \in \states} \left[\frac{\Pr(s', \obs \midd s, \act)}{\Pr(\obs \midd s, \act)} Q(\bel_{s'}, \act')\right].
\]
Using this, we find
\begin{align*}
\sum_{s \in \states} \left[ b(s) \Pr(\obs \midd s, \act)  Q(\bel_{s,\act,\obs}, \act') \right] 
&\leq \sum_{s \in \states} \left[ b(s) \Pr(\obs \midd s, \act) \sum_{s' \in \states} \left[\frac{\Pr(s', \obs \midd s, \act)}{\Pr(\obs \midd s, \act)} Q(\bel_{s'}, \act')\right]\right] \\
&= \sum_{s \in \states}  \sum_{s' \in \states} \left[ b(s)  \Pr(s', \obs \midd s, \act)Q(\bel_{s'},  \act')\right] \\
&=  \sum_{s' \in \states}  \left[ \sum_{s \in \states} [b(s)  \Pr(s', \obs \midd s, \act)]  Q(\bel_{s'}, \act')\right] \\
&=  \sum_{s' \in \states}  \left[ \Pr(s', \obs \midd b, \act)  Q(\bel_{s'}, \act')\right]. 
\end{align*}

Hence, we get
\begin{align*}
H_\BIB Q(b, \act) &= R(b,\act) + \gamma \sum_{\obs \in \observations}  \max\limits_{\act' \in \actions} \sum_{s \in \states} \left[ b(s) \Pr(\obs \midd s, \act)  Q(\bel_{s,\act,\obs}, \act') \right] \\
 & \leq R(b,\act) + \gamma \sum_{\obs \in \observations}  \max\limits_{\act' \in \actions} \sum_{s' \in \states}  \left[ \Pr(s', \obs \midd b, \act)  Q(\bel_{s'}, \act')\right] = H_\FIB Q(b, \act).
\end{align*}
Since $\qfib$ is convex (by \cref{lem:bibconvex}), applying the result $H_{\BIB}Q \leq H_{\FIB}Q$ to $\qfib$ shows that $H_\BIB \qfib \leq H_\FIB \qfib = \qfib$. By \cref{lem:ineq}, it now follows that $\qbib \leq \qfib$, i.e.\ that $\qbib(b, \act) \leq \qfib(b,\act)$ for all $b \in \distr{\states}$ and $\act \in \actions$.
\end{proof}

\subsection{Optimized Tighter Informed Bound}

Now we turn to \OBIB. We again start by showing that the fixed point $\qobib$ exists and is unique. The techniques are mostly similar to the techniques used in the proof of \cref{thm:qbib_exists}, but we need a second application of \cref{lem:contraction} to deal with the minimization over the weights.

\begin{theorem} 
$H_\OBIB$ is a contraction operator w.r.t.\ the  $\infty$-norm with Lipschitz constant $\gamma < 1$, and $H_\OBIB$ has a unique fixed point $\qobib$. 
\end{theorem}

\begin{proof}
Let $Q, Q' \colon \distr{\states} \times \actions \rightarrow \mathbb{R}$ be any $Q$-functions. Fix $b \in \distr{\states}$ and $a \in \actions$. Then we have
\begin{align*}
    | H_\OBIB  Q(b,a) - H_\OBIB Q'(b,a) |  
     &= \left| \gamma \!\sum_{\obs \in \observations} \max_{\act' \in \actions} [\Pr(\obs \midd b, \act) \!\! \min_{w \in \weightset_{b,a,o}} \!\!\Qinner{}{w}{}{\act'}] - \gamma \!\sum_{\obs \in \observations} \max_{\act' \in \actions} [\Pr(\obs \midd b, \act) \!\! \min_{w \in \weightset_{b,a,o}} \!\!\Qinnerprime{}{w}{}{\act'}]\right| \\
    &\leq  \gamma \sum_{\obs \in \observations}  \left|  \max_{\act' \in \actions} [\Pr(\obs \midd b, \act) \!\! \min_{w \in \weightset_{b,a,o}} \!\!\Qinner{}{w}{}{\act'}] - \max_{\act' \in \actions} [\Pr(\obs \midd b, \act) \!\! \min_{w \in \weightset_{b,a,o}}  \!\!\Qinnerprime{}{w}{}{\act'}] \right|  \\
    &\leq  \gamma \sum_{\obs \in \observations} \max\limits_{\act' \in \actions}  \Pr(\obs \midd b, \act) \left| \min_{w \in \weightset_{b,a,o}} \!\!\Qinner{}{w}{}{\act'}- \min_{w \in \weightset_{b,a,o}} \!\!\Qinnerprime{}{w}{}{\act'} \right|  \\
    &\leq  \gamma \sum_{\obs \in \observations} \max\limits_{\act' \in \actions}  \Pr(\obs \midd b, \act) \sup_{w \in \weightset_{b,a,o}}   \left| \Qinner{}{w}{}{\act'} - \Qinnerprime{}{w}{}{\act'} \right|  \\[-6pt]
    &=  \gamma \sum_{\obs \in \observations} \max\limits_{\act' \in \actions}  \Pr(\obs \midd b, \act) \sup_{w \in \weightset_{b,a,o}}   \Bigg| \sum_{b' \in \Bsao} w(b') \left(Q(b', a') - Q'(b',a')\right)\Bigg|  \\
    &\leq  \gamma \sum_{\obs \in \observations} \max\limits_{\act' \in \actions}  \Pr(\obs \midd b, \act) \sup_{w \in \weightset_{b,a,o}} \sum_{b' \in \Bsao}   w(b') \left|  \left(Q(b', a') - Q'(b',a')\right)\right|  \\
    &\leq  \gamma \sum_{\obs \in \observations} \max\limits_{\act' \in \actions}  \Pr(\obs \midd b, \act) \sup_{w \in \weightset_{b,a,o}} \sum_{b' \in \Bsao}   w(b') \left\|   Q- Q' \right\|_{\infty}   \\
    &= \gamma \left\|   Q- Q' \right\|_{\infty} \sum_{\obs \in \observations}\Pr(\obs \midd b, \act) \sup_{w \in \weightset_{b,a,o}} \sum_{b' \in \Bsao}   w(b') \\
    &= \gamma \left\|   Q- Q' \right\|_{\infty} \sum_{\obs \in \observations}  \Pr(\obs \midd b, \act) = \gamma \left\|   Q- Q' \right\|_{\infty}.
\end{align*}
The first three inequalities follow from the triangle inequality, and applying \cref{lem:contraction} twice. Next, we use \cref{eq:PSA},  the triangle inequality, the definition of the $\infty$-norm, and \cref{lem:weightfun}. By the definition of the $\infty$-norm, it follows that
\[
\left\|   H_\OBIB Q - H_\OBIB Q' \right\|_{\infty} \leq \max_{(b,a) \in \distr{\states} \times \actions} | H_\OBIB Q(b,a) - H_\OBIB Q'(b,a) | \leq \gamma \left\|   Q- Q' \right\|_{\infty},
\]
which shows that  $H_\OBIB$ is a contraction operator with Lipschitz constant $\gamma$. Since $H_\OBIB$ can be computed using only the finite set of one-step beliefs, \cref{lem:banach} now implies that $H_\OBIB$ has a unique fixed point $\qobib$.
\end{proof}

\noindent Next, we show that $\qobib$ is convex in $b$.

\begin{theorem}  
Let $\act \in \actions$ be given. Then  $b \mapsto \qobib(b,a)$ is convex.
\end{theorem}

\begin{proof}
Let $Q$ be any $Q$-function. We first show that for given $\act \in \actions$, the function $b \mapsto H_\OBIB Q(b, \act)$ is convex using the definition of convexity. Let $b, b' \in \distr{\states}$ and $\lambda \in [0,1]$ be given. Write $b'' = \lambda b + (1-\lambda)b'$. Then
\[
\Pr(o \midd b'', a) = \sum_{s \in \states} b''(s) \Pr(\obs \midd s, \act) = \lambda \sum_{s \in \states} b(s) \Pr(\obs \midd s, \act)  + (1-\lambda)\sum_{s \in \states} b'(s) \Pr(\obs \midd s, \act) = \lambda \Pr(\obs \midd b, \act)  + (1-\lambda) \Pr(\obs \midd b', \act).
\]
We now show that if $w \in \weightset_{b, a, o}$ and $w' \in \weightset_{b', a, o}$ are weight functions for $\bel_{b,a,o}$ and $\bel_{b',a,o}$, then $w'' =  \frac{\lambda \Pr(\obs \midd b, \act)}{\Pr(o \midd b'', a)} w + \frac{(1-\lambda) \Pr(\obs \midd b', \act)}{\Pr(o \midd b'', a)} w'$ is a weight function for $\bel_{b'', a, o}$. This follows from
\begin{align*}
\sum_{\tilde{b} \in \mathcal{B}}   w''(\tilde{b} ) \tilde{b} (s) &= \frac{\lambda \Pr(\obs \midd b, \act)}{\Pr(o \midd b'', a)}  \sum_{\tilde{b} \in \mathcal{B}}   w(\tilde{b} ) \tilde{b} (s) + \frac{(1-\lambda) \Pr(\obs \midd b', \act)}{\Pr(o \midd b'', a)} \sum_{\tilde{b} \in \mathcal{B}}   w'(\tilde{b} ) \tilde{b} (s) \\ &= \frac{\lambda \Pr(\obs \midd b, \act)}{\Pr(o \midd b'', a)}  \bel_{b,a,o}(s) + \frac{(1-\lambda) \Pr(\obs \midd b', \act)}{\Pr(o \midd b'', a)} \bel_{b',a,o}(s)
 \\ &= \frac{\lambda \Pr(\obs \midd b, \act)}{\Pr(o \midd b'', a)} \frac{\sum_{s\in \states} b(s) \Pr(s', \obs \midd s, \act)}{\Pr(\obs \midd b, \act)} + \frac{(1-\lambda) \Pr(\obs \midd b', \act)}{\Pr(o \midd b'', a)} \frac{\sum_{s\in \states} b'(s) \Pr(s', \obs \midd s, \act)}{\Pr(\obs \midd b', \act)}
 \\ &=  \frac{\lambda \sum_{s\in \states} b(s) \Pr(s', \obs \midd s, \act) + (1-\lambda) \sum_{s\in \states} b'(s) \Pr(s', \obs \midd s, \act)}{\Pr(o \midd b'', a)}.
  \\ &=  \frac{\sum_{s\in \states} b''(s) \Pr(s', \obs \midd s, \act)}{\Pr(o \midd b'', a)} = \bel_{b'', a, o}(s).
\end{align*}
Moreover, we have
\begin{align*}
\Pr(o \midd b'', a) \Qinner{}{w''}{}{\act'} &= \Pr(o \midd b'', a)  \sum_{\tilde{b} \in \mathcal{B}} w''(\tilde{b}) Q(\tilde{b}, a') \\ &= \lambda\Pr(\obs \midd b, \act) \sum_{\tilde{b} \in \mathcal{B}} w(\tilde{b}) Q(\tilde{b}, a') + (1-\lambda)\Pr(\obs \midd b', \act) \sum_{\tilde{b} \in \mathcal{B}} w'(\tilde{b}) Q(\tilde{b}, a') \\ &= \lambda \Pr(\obs \midd b, \act) \Qinner{}{w}{}{\act'} + (1-\lambda)\Pr(\obs \midd b', \act) \Qinner{}{w'}{}{\act'}.
\end{align*}
Hence, for any  $w \in \weightset_{b, a, o}$ and $w' \in \weightset_{b', a, o}$, there exists a $w'' \in \weightset_{b'', a, o}$ such that 
\[
\Pr(o \midd b'', a) \Qinner{}{w''}{}{\act'}  = \lambda \Pr(\obs \midd b, \act) \Qinner{}{w}{}{\act'} + (1-\lambda)\Pr(\obs \midd b', \act) \Qinner{}{w'}{}{\act'}.
\]
By taking $w \in \argmin\limits_{w \in \weightset_{b,a,o}}  \Qinner{}{w}{}{\act'}$ and $w' \in \argmin\limits_{w' \in \weightset_{b',a,o}}  \Qinner{}{w'}{}{\act'}$, this implies 
\[
\Pr(o \midd b'', a) \!\! \min\limits_{w'' \in \weightset_{b'',a,o}}\!\!   \Qinner{}{w''}{}{\act'}  \leq \lambda \Pr(\obs \midd b, \act) \!\min\limits_{w \in \weightset_{b,a,o}}\!\!   \Qinner{}{w}{}{\act'} + (1-\lambda)\Pr(\obs \midd b', \act) \!\min\limits_{w' \in \weightset_{b',a,o}}\!\!  \Qinner{}{w'}{}{\act'}.
\]
This in turn implies that
\begin{align*}
\max\limits_{\act' \in \actions} \left[\Pr(o \midd b'', a) \!\! \min\limits_{w'' \in \weightset_{b'',a,o}}\!\!   \Qinner{}{w''}{}{\act'} \right]  &\leq \max\limits_{\act' \in \actions} \left[\lambda \Pr(\obs \midd b, \act) \!\min\limits_{w \in \weightset_{b,a,o}}\!\!   \Qinner{}{w}{}{\act'} + (1-\lambda)\Pr(\obs \midd b', \act) \!\min\limits_{w' \in \weightset_{b',a,o}}\!\!  \Qinner{}{w'}{}{\act'}\right] \\
 &\leq \lambda \max\limits_{\act' \in \actions} \left[\Pr(\obs \midd b, \act) \!\min\limits_{w \in \weightset_{b,a,o}}\!\!   \Qinner{}{w}{}{\act'}\right] + (1-\lambda)\max\limits_{\act' \in \actions} \left[\Pr(\obs \midd b', \act) \!\min\limits_{w' \in \weightset_{b',a,o}}\!\!  \Qinner{}{w'}{}{\act'}\right]
\end{align*}
Summing over all $o \in \observations$ and using that $R(b'', a) = \lambda R(b,a) + (1-\lambda)R(b',a)$ now yields
\begin{align*}
H_\OBIB &Q(b'', \act) = R(b'', a) + \gamma \sum_{o \in \observations} \max\limits_{\act' \in \actions} \left[\Pr(o \midd b'', a) \!\! \min\limits_{w'' \in \weightset_{b'',a,o}}\!\!   \Qinner{}{w''}{}{\act'} \right] \\
 &\leq \lambda \left[R(b, a) \! +\!  \gamma\! \sum_{o \in \observations} \max\limits_{\act' \in \actions} \left[\Pr(\obs \midd b, \act) \!\min\limits_{w \in \weightset_{b,a,o}}\!\!   \Qinner{}{w}{}{\act'}\right]\right] + (1\!-\!\lambda)\left[R(b', a) \! +\!  \gamma \! \sum_{o \in \observations}  \max\limits_{\act' \in \actions} \left[\Pr(\obs \midd b', \act) \!\min\limits_{w' \in \weightset_{b',a,o}}\!\!  \Qinner{}{w'}{}{\act'}\right]\right] \\
 &= \lambda H_\OBIB Q(b, \act) +  (1\!-\!\lambda) H_\OBIB Q(b', \act).
\end{align*}
Since $b, b' \in \distr{\states}$ and $\lambda \in [0,1]$ were given arbitrarily, we conclude that the function $b \mapsto H_\OBIB Q(b, \act)$ is convex.
Since $\qobib$ satisfies $\qobib = H_\OBIB \qobib$, this in particular implies that $b \mapsto \qobib(b,a)$ is convex for any given $\act \in \actions$.
\end{proof}

\newpage

\noindent Next, we show that \OBIB is a sound upper bound. 

\begin{theorem} \label{thm:obib_sound}
For any belief $b \in \distr{\states}$ and any $\act \in \actions$, it holds that $\qobib(b, \act) \geq \qpomdp(b,\act)$.
\end{theorem}

\begin{proof}
Let $Q$ be any convex $Q$-function. We first show that $H_{\POMDP}Q \leq H_{\OBIB}Q$.
 By definition, a weight function $w \in \weightset_{b,a,o}$ writes the belief $\bel_{b,\act,\obs}$ as a convex combination of state beliefs:
$
\bel_{b, \act, \obs} = \sum_{b' \in \Bsao} \left[w(b') b'\right].$
Hence, the convexity of $Q$ implies
\[
Q(\bel_{b, \act, \obs}, \act') \leq \sum_{b' \in \Bsao} \left[ w(b') Q(b', \act')\right] = \Qinner{}{w}{}{\act'}.
\]
Since this holds for any $w \in \weightset_{b,a,o}$, we find $Q(\bel_{b, \act, \obs}, \act') \leq \min\limits_{w \in \weightset_{b,a,o}} \Qinner{}{w}{}{\act'}$.
Hence, we get
\begin{align*}
H_\POMDP Q(b, \act) 
& =  R(b,\act) + \gamma \sum_{\obs \in \observations}  \max\limits_{\act' \in \actions} \left[ \Pr(\obs \midd b, \act)  Q(\bel_{b,\act,\obs}, \act') \right]\\ &\leq  R(b,\act) + \gamma \sum_{\obs \in \observations} \max_{\act' \in \actions}  \left[\Pr(\obs \midd b, \act) \!\min\limits_{w \in \weightset_{b,a,o}}\!\!   \Qinner{}{w}{}{\act'}\right] 
  = H_\OBIB Q(b, \act).
\end{align*}
Since $\qpomdp$ is convex, applying the result $H_{\POMDP}Q \leq H_{\OBIB}Q$ to $\qpomdp$ shows that $\qpomdp = H_{\POMDP}\qpomdp \leq H_{\OBIB} \qpomdp$. By \cref{lem:ineq}, it now follows that $\qpomdp \leq \qobib$, i.e.\ that $\qpomdp(b, \act) \leq \qobib(b,\act)$ for all $b \in \distr{\states}$ and $\act \in \actions$.
\end{proof}

\noindent Finally, we show that \OBIB is tighter than \BIB and \EBIB.

\begin{theorem} \label{thm:obib_bib}
For any belief $b \in \distr{\states}$ and any $\act \in \actions$, it holds that $\qobib(b, \act) \leq \qbib(b,\act)$ and $\qobib(b, \act) \leq \qebib(b,\act)$ .
\end{theorem}

\begin{proof}
Let $Q$ be any $Q$-function. We first show that $H_{\BIB}Q \geq H_{\OBIB}Q$. Recall that we can write \begin{equation*}
    H_{\BIB}Q(b, \act)  = R(b,\act) + \gamma \sum_{\obs \in \observations}  \max\limits_{\act' \in \actions} \left[ \Pr(\obs \midd b, \act)\Qinner{\BIB}{w_{b,\act,\obs}}{t{-}1}{\act}\right],
\end{equation*}
with $w_{b,\act,\obs}(\bel_{s,\act,\obs}) = \frac{b(s) \Pr( \obs \midd s, \act)}{\Pr(\obs \midd b, \act)}\in \weightset_{b,a,o} $. Hence, we get $\Qinner{}{w_{b,\act,\obs}}{t{-}1}{\act} \geq\min\limits_{w \in \weightset_{b,a,o}}\!\!   \Qinner{}{w}{}{\act'} $, so
\begin{align*}
H_\BIB Q(b, \act) 
& =  R(b,\act) + \gamma \sum_{\obs \in \observations}  \max\limits_{\act' \in \actions} \left[ \Pr(\obs \midd b, \act)\Qinner{\BIB}{w_{b,\act,\obs}}{t{-}1}{\act}\right] \\ &\geq  R(b,\act) + \gamma \sum_{\obs \in \observations} \max_{\act' \in \actions}  \left[\Pr(\obs \midd b, \act) \!\min\limits_{w \in \weightset_{b,a,o}}\!\!   \Qinner{}{w}{}{\act'}\right] 
  = H_\OBIB Q(b, \act).
\end{align*}
Applying the result $H_{\OBIB}Q \leq H_{\BIB}Q$ to $\qobib$ shows that $\qobib = H_{\OBIB}\qobib \leq H_{\BIB} \qobib$. By \cref{lem:ineq}, it now follows that $\qobib \leq \qbib$. The proof that  $\qobib \leq \qebib$ is analogous, except that we now use that the entropy-based  weights $\went_{b,a,o}$ satisfy $\went_{b,a,o}\in \weightset_{b,a,o}$ and hence $\Qinner{}{\went_{b,\act,\obs}}{t{-}1}{\act} \geq\min\limits_{w \in \weightset_{b,a,o}}\!\!   \Qinner{}{w}{}{\act'} $,
\end{proof}

\subsection{Entropy-Based Tighter Informed Bound}

Now we turn to \EBIB. We again start by showing that the fixed point $\qebib$ exists and is unique.

\begin{theorem} 
$H_\EBIB$ is a contraction operator w.r.t.\ the  $\infty$-norm with Lipschitz constant $\gamma < 1$, and $H_\EBIB$ has a unique fixed point $\qebib$. 
\end{theorem}

\begin{proof}
Let $Q, Q' \colon \distr{\states} \times \actions \rightarrow \mathbb{R}$ be any $Q$-functions. Fix $b \in \distr{\states}$ and $a \in \actions$. Then we have
\begin{align*}
    | H_\EBIB  Q(b,a) - H_\EBIB Q'(b,a) |  
     &= \left| \gamma \!\sum_{\obs \in \observations} \max_{\act' \in \actions} [\Pr(\obs \midd b, \act) \Qinner{}{\went_{b,\act,\obs}}{t{-}1}{\act'}] - \gamma \!\sum_{\obs \in \observations} \max_{\act' \in \actions} [\Pr(\obs \midd b, \act) \Qinner{}{\went_{b,\act,\obs}}{t{-}1}{\act'}]\right| \\
    &\leq  \gamma \sum_{\obs \in \observations}  \left|  \max_{\act' \in \actions} [\Pr(\obs \midd b, \act) \Qinner{}{\went_{b,\act,\obs}}{t{-}1}{\act'}] - \max_{\act' \in \actions} [\Pr(\obs \midd b, \act) \Qinner{}{\went_{b,\act,\obs}}{t{-}1}{\act'}] \right|  \\
    &\leq  \gamma \sum_{\obs \in \observations} \max\limits_{\act' \in \actions}  \Pr(\obs \midd b, \act)  \left|  \Qinner{}{\went_{b,\act,\obs}}{t{-}1}{\act'}] - \Qinner{}{\went_{b,\act,\obs}}{t{-}1}{\act'} \right|  \\[-6pt]
\displaybreak
    &=  \gamma \sum_{\obs \in \observations} \max\limits_{\act' \in \actions}  \Pr(\obs \midd b, \act)   \Bigg| \sum_{b' \in \Bsao} \went_{b,\act,\obs}(b') \left(Q(b', a') - Q'(b',a')\right)\Bigg|  \\
    &\leq  \gamma \sum_{\obs \in \observations} \max\limits_{\act' \in \actions}  \Pr(\obs \midd b, \act) \sum_{b' \in \Bsao}    \went_{b,\act,\obs}(b') \left|  \left(Q(b', a') - Q'(b',a')\right)\right|  \\
    &\leq  \gamma \sum_{\obs \in \observations} \max\limits_{\act' \in \actions}  \Pr(\obs \midd b, \act) \sum_{b' \in \Bsao}   \went_{b,\act,\obs}(b')  \left\|   Q- Q' \right\|_{\infty}   \\
    &= \gamma \left\|   Q- Q' \right\|_{\infty} \sum_{\obs \in \observations}\Pr(\obs \midd b, \act) \sum_{b' \in \Bsao}   \went_{b,\act,\obs}(b') \\
    &= \gamma \left\|   Q- Q' \right\|_{\infty} \sum_{\obs \in \observations}  \Pr(\obs \midd b, \act) = \gamma \left\|   Q- Q' \right\|_{\infty}.
\end{align*}
The first three inequalities follow from the triangle inequality, and applying \cref{lem:contraction} twice. Next, we use \cref{eq:PSA},  the triangle inequality, the definition of the $\infty$-norm, and \cref{lem:weightfun}. 

By the definition of the $\infty$-norm, it follows that
\[
\left\|   H_\EBIB Q - H_\EBIB Q' \right\|_{\infty} \leq \max_{(b,a) \in \distr{\states} \times \actions} | H_\EBIB Q(b,a) - H_\EBIB Q'(b,a) | \leq \gamma \left\|   Q- Q' \right\|_{\infty},
\]
which shows that  $H_\EBIB$ is a contraction operator with Lipschitz constant $\gamma$. Since $H_\EBIB$ can be computed using only the finite set of one-step beliefs, \cref{lem:banach} now implies that $H_\EBIB$ has a unique fixed point $\qebib$.
\end{proof}

\noindent The next theorem shows that \EBIB is a sound upper bound.

\begin{theorem}
For any belief $b \in \distr{\states}$ and any $\act \in \actions$, it holds that $\qebib(b, \act) \geq \qpomdp(b,\act)$.
\end{theorem}

\begin{proof}
By \cref{thm:obib_sound} and \cref{thm:obib_bib}, we have that $\qebib(b, \act) \geq \qobib(b,\act) \geq \qpomdp(b,\act)$.
\end{proof}

\noindent Finally, we show that \EBIB is tighter than \OBIB.

\begin{theorem} 
For any belief $b \in \distr{\states}$ and any $\act \in \actions$, it holds that $\qebib(b, \act) \leq \qfib(b,\act)$.
\end{theorem}
\begin{proof}
Let $Q$ be any convex $Q$-function. We first show that $H_{\EBIB}Q \leq H_{\FIB}Q$.
 We can write any belief $b'$ as a convex combination of state beliefs:
$
b' = \sum_{s' \in \states} \left[b'(s')\bel_{s'}\right].$
Hence, the convexity of $Q$ implies
\[
Q(b', \act') \leq \sum_{s' \in \states} \left[ b'(s') Q(\bel_{s'}, \act')\right].
\]
Using this, we find
\begin{align*}
\Qinner{}{\went_{b,\act,\obs}}{t{-}1}{\act'} 
&= \sum_{b \in \Bsao} \left[ \went_{b,\act,\obs}(b')Q(b', a') \right] \\
&= \sum_{b \in \Bsao} \left[ \went_{b,\act,\obs}(b') \sum_{s' \in \states} \left[ b'(s') Q(\bel_{s'}, \act')\right]\right] \\
&= \sum_{b \in \Bsao} \sum_{s' \in \states} \left[ \went_{b,\act,\obs}(b') b'(s') Q(\bel_{s'}, \act')\right] \\
&=  \sum_{s' \in \states}  \left[  \sum_{b \in \Bsao} \left[ \went_{b,\act,\obs}(b') b'(s') \right] Q(\bel_{s'}, \act')\right] \\
&=  \sum_{s' \in \states}  \left[ \bel_{b,a,o}(s')  Q(\bel_{s'}, \act')\right] \\ 
&=  \sum_{s' \in \states}  \left[ \frac{\Pr(s', \obs \midd b, \act)}{\Pr(o \midd b,\act)} Q(\bel_{s'}, \act')\right]. 
\end{align*}

Hence, we get
\begin{align*}
H_\EBIB Q(b, \act) &= R(b,\act) + \gamma \sum_{\obs \in \observations}  \max\limits_{\act' \in \actions}  \left[ \Pr(o \midd b,a) \Qinner{}{\went_{b,\act,\obs}}{t{-}1}{\act'}  \right] \\
 & \leq R(b,\act) + \gamma \sum_{\obs \in \observations}  \max\limits_{\act' \in \actions} \sum_{s' \in \states}  \left[ \Pr(s', \obs \midd b, \act)  Q(\bel_{s'}, \act')\right] = H_\FIB Q(b, \act).
\end{align*}
Since $\qfib$ is convex (by \cref{lem:bibconvex}), applying the result $H_{\EBIB}Q \leq H_{\FIB}Q$ to $\qfib$ shows that $H_\EBIB \qfib \leq H_\FIB \qfib = \qfib$. By \cref{lem:ineq}, it now follows that $\qebib \leq \qfib$, i.e.\ that $\qebib(b, \act) \leq \qfib(b,\act)$ for all $b \in \distr{\states}$ and $\act \in \actions$.
\end{proof}

\subsection{Value Iteration}

\noindent We approximate the fixed points of our Bellman operators using value iteration. The following result shows a termination criterion for the value iteration, by showing how close the actual solution and the final iteration are based on the distance between the last two iterations.

\begin{lemma}
Let $H$ be any contraction operator with Lipschitz constant $\gamma$, and let $Q^*$ be its fixed point. Then $\|Q^* - HQ\|_{\infty}  \leq \frac{\gamma}{1-\gamma}\|HQ - Q\|_{\infty}$.
\end{lemma}

\begin{proof}
Let $H$ be any contraction operator with Lipschitz constant $\gamma$, and let $Q^*$ be its fixed point. Then 
\[\|Q^* - Q\|_{\infty} \leq \|Q^* - HQ\|_{\infty} + \|HQ - Q\|_{\infty} = \|HQ^* - HQ\|_{\infty} + \|HQ - Q\|_{\infty} \leq \gamma\|Q^* - Q\|_{\infty} + \|HQ - Q\|_{\infty},\]
using the triangle inequality, the fact that $Q^*$ is a fixed point of $H$ and the fact that $H$ is a contraction operator with Lipschitz constant $\gamma$. This implies that $\|Q^* - Q\|_{\infty} \leq \frac1{1-\gamma}\|HQ - Q\|_{\infty}$, and hence that $\|Q^* - HQ\|_{\infty} \leq \gamma\|Q^* - Q\|_{\infty}  \leq \frac{\gamma}{1-\gamma}\|HQ - Q\|_{\infty}$.
\end{proof}

\section{Analyzing increased delays}
\label{ap:3stepdelay}

In this section, we give an intuitive analysis of the effect of considering longer observation delays.
We start with a concrete variant of \BIB where we assume an agent observes its state after 3 time steps instead of 2, which we denote as \BIBThree.
We first introduce some notation: similarly to $\bel_{s,a,o}$, we denote $\bel_{s,a,o,a',o'} = \bel_{\bel_{s,a,o},a',o'}$ as a \emph{two-step belief}, i.e. the belief after taking actions $a$ and $a'$ from state $s$ and observing $o$ and $o'$.
Furthermore, let $\Bbao$ denote the set of all such beliefs plus the initial belief, with $|\Bbao| \in \bigO(|\states||\actions|^2|\observations|^2)$.
Then, we can write the Bellman operator for \BIBThree as follows:
\begin{equation}
    H_\BIBThree Q(b,a) = R(b,a) + \gamma \sum_{\obs \in \observations} \max_{\act' \in \actions} \sum_{\obs' \in \observations} \max_{\act'' \in \actions} \Bigg[
        \sum_{s \in \states} b(s) \Pr(\obs \midd s,\act) \bigg[ 
            \rewards(\bel_{s,\act,\obs}, \act') + \gamma \Pr(\obs' \midd \bel_{s,\act,\obs}, \act') Q(\bel_{s,\act,\obs,\act',\obs'}, \act'')
        \bigg] 
    \Bigg]
\end{equation}
\noindent We omit a proof that this operator has a unique fixed point, but this can be proven analogously to the other bounds.
In the equation, the actions $a'$ and $a''$ are chosen independently of the state $s$, as opposed to only an action $a'$ for \BIB.
Thus, we increased the observation delay from 2 time steps to 3.

Assuming all probability calculations can be computed in $O(1)$ amortized time (e.g. they are cached), then the computational complexity of a single operation is $\bigO(|\Bbao|) \in \bigO(|\states||\actions|^2|\observations|^2)$, which is a factor $\bigO(|\actions||\observations|)$ larger than for \BIB.
Furthermore, let us assume we approximate this bound using value iterations, similarly to \cref{alg:BIB}.
Then, a single value iteration step requires iterating over all beliefs in $\Bbao$ instead of $\Bsao$.
Thus, the total complexity for precomputations is $\bigO(|\Bbao|^2|\actions|h) \in \bigO(|\states|^2|\actions|^5|\observations|^4h)$, a factor $\bigO(|\actions|^2|\observations|^2)$ larger than for \BIB.

It is easy to see that this pattern extends to larger delays, i.e. given a delay $n\geq2$, the complexity for precomputations is given as $\bigO(|\states|^2|\actions|^{1+2(n-1)}|\observations|^{2(n-1)}h)$.
Since beliefs in $\Bbao$ may be identical, computational costs may be lower in practice, but such analysis falls outside the scope of this work.
\end{document}

%% file: Tables/Complexities.tex
\begin{tabular}{@{}lcl@{}}
\toprule
\textbf{Bound} & Bellman Operation & Precomputations\\
 \cmidrule{1-1} \cmidrule(l){2-2} \cmidrule(l){3-3} 
\FIB  & $\bigO(|\states| |\actions| |\observations|)$ &  $\bigO(|\states|^2|\actions|^2|\observations| h)$ \\
\BIB  & $\bigO(|\states||\actions||\observations|)$ & $\bigO(|\Bsao| |\states| |\actions|^2|\observations| h)$ \\
\OBIB  & $\bigO(|\actions||\observations|L)$ & $\bigO(|\Bsao||\actions|^2|\observations|Lh)$ \\
\EBIB  & $\bigO(|\observations|L)$ & $\bigO(|\Bsao||\actions||\observations|L + |\Bsao|^2|\actions|^2|\observations|h)$ \\
\bottomrule
\end{tabular}

%% file: Tables/Eval_Bounds.tex
\begin{tabular}{@{}lcrrrrrcrrcrrcrrcrrcrr@{}}
\toprule
&& \multicolumn{5}{c}{\textbf{Environment Properties}} && \multicolumn{5}{c}{\textbf{Baselines}} && \multicolumn{8}{c}{\textbf{Our methods}} \\
\cmidrule{3-7} \cmidrule{9-13} \cmidrule{15-22}
 \textbf{Environments} && $|\states|$ & $|\actions|$ & $|\observations|$ & $|\Bsao|$ & $|\Bbao|$ && \multicolumn{2}{c}{$\SARSOP$} && \multicolumn{2}{c}{\FIB} && \multicolumn{2}{c}{\BIB} && \multicolumn{2}{c}{\EBIB} && \multicolumn{2}{c}{\OBIB} \\
\cmidrule{1-1} \cmidrule{3-7} \cmidrule{9-10} \cmidrule{12-13} \cmidrule{15-16} \cmidrule{18-19} \cmidrule{21-22}
\custom &  & 3 & 3 & 1 & 6 & 8 &
& \underline{0.50} & <1s &  & 0.76 & <1s &  & 0.61 & <1s &  & \textbf{0.50} & <1s &  & \textbf{0.50} & <1s \\
\tiger &  & 2 & 3 & 2 & 3 & 5 &  
& \underline{19.4} & <1s &  & 87.2 & <1s &  & 49.6 & <1s &  & \textbf{40.5} & <1s &  & \textbf{40.5} & <1s \\
\gridenv &  & 36 & 5 & 36 & 152 & 722 &  
& {6.42} & TO &  & 8.31 & <1s &  & 8.15 & 4s &  & 7.25 & 5s &  & \textbf{7.20} & 64s \\
\rocksample (5,3) &  & 201 & 8 & 3 & 202 & 210 &  
& \underline{16.9} & 1s &  & \textbf{18.3} & <1s &  & \textbf{18.3} & <1s &  & \textbf{18.3} & <1s &  & \textbf{18.3} & <1s \\
\rocksample (7,8) &  & 13k & 13 & 3 & 13k & 13k &  
& {20.9} & TO &  & 28.5 & 106s &  & \textbf{27.2} & 413s &  & \textbf{27.2} & 728s &  & \textbf{27.2} & 741s \\
\koutofn (2) &  & 16 & 9 & 16 & 61 & 230 &  
& {-1.75} & TO &  & -1.24 & <1s &  & \textbf{-1.52} & 4s &  & \textbf{-1.52} & 4s &  & \textbf{-1.52} & 15s \\
\koutofn (3) &  & 64 & 27 & 64 & 499 & 4.8k &  
& {-2.63} & TO &  & -1.89 & <1s &  & -2.28 & 9s &  & -2.28 & 14s &  & \textbf{-2.29} & 473s \\
\aloha (30) &  & 90 & 29 & 90 & 2.5k & 202k &  
& {389} & TO &  & 394 & 5s &  & \textbf{392} & 24s &  & \textbf{392} & 923s &  & \textbf{392} & TO \\
\tagenv &  & 842 & 5 & 30 & 2.4k & 6.3k &  
& {-10.8} & TO &  & -4.75 & 5s &  & -5.58 & 22s &  & -5.57 & 35s &  & \textbf{-5.64} & 575s \\
\tigergrid &  & 36 & 5 & 36 & 1.4k & 100k &  
& {2.28} & TO &  & 2.73 & <1s &  & 2.58 & 31s &  & 2.57 & 270s &  & \textbf{2.58} & TO \\
\hallwayone &  & 60 & 5 & 60 & 2.2k & 147k &  
& {1.00} & TO &  & 1.29 & 2s &  & 1.19 & 28s &  & \textbf{1.17} & 398s &  & 1.18 & TO \\
\hallwaytwo &  & 92 & 5 & 92 & 3.4k & 229k &  
& {0.34} & TO &  & 0.98 & 5s &  & 0.89 & 52s &  & \textbf{0.88} & 777s &  & 0.89 & TO \\
\pentagon &  & 212 & 4 & 212 & 6.0k & 447k &  
& {0.33} & TO &  & \textbf{0.38} & 4s &  & \textbf{0.38} & 24s &  & \textbf{0.38} & 645s &  & \textbf{0.38} & TO \\
\fourth &  & 1.1k & 4 & 1.1k & 29k & 2125k &  
& {0.06} & TO &  & \textbf{0.09} & 102s &  & \textbf{0.09} & 432s &  & \textbf{0.09} & TO &  & \textbf{0.09} & TO \\
\bottomrule
\end{tabular}

%% file: Tables/Eval_Sarsop_Large.tex
\begin{tabular}{@{}lrrrrrrrrr@{}}
\toprule
& \multicolumn{3}{c}{\textbf{600s}} & \multicolumn{3}{c}{\textbf{1200s}} & \multicolumn{3}{c}{\textbf{3600s}}\\
\cmidrule(l){2-4} \cmidrule(l){5-7} \cmidrule(l){8-10}
\textbf{Environments} & \multicolumn{1}{c}{\FIB} & \multicolumn{1}{c}{\BIB} & \multicolumn{1}{c}{\EBIB} & \multicolumn{1}{c}{\FIB} & \multicolumn{1}{c}{\BIB} & \multicolumn{1}{c}{\EBIB} & \multicolumn{1}{c}{\FIB} & \multicolumn{1}{c}{\BIB} & \multicolumn{1}{c}{\EBIB} \\
\cmidrule{1-1} \cmidrule(l){2-4} \cmidrule(l){5-7} \cmidrule(l){8-10}
\gridenv & 0.16 & 0.16 & \textbf{0.07} & 0.12 & 0.12 & \textbf{0.06} & 0.09 & 0.08 & \textbf{0.06} \\
\rocksample (7,8) & \textbf{0.23} & 0.27 & N/A & 0.19 & \textbf{0.18} & 0.21 & \textbf{0.14} & 0.15 & \textbf{0.14} \\
\koutofn (3) & 0.21 & \textbf{0.12} & \textbf{0.12} & 0.20 & \textbf{0.11} & \textbf{0.11} & 0.18 & 0.10 & \textbf{0.09} \\
\tagenv & 0.44 & \textbf{0.40} & \textbf{0.40} & 0.43 & \textbf{0.39} & \textbf{0.39} & 0.41 & \textbf{0.37} & \textbf{0.37} \\
\tigergrid & 0.11 & \textbf{0.09} & \textbf{0.09} & 0.11 & \textbf{0.08} & \textbf{0.08} & 0.10 & \textbf{0.08} & \textbf{0.08} \\
\hallwayone & 0.22 & \textbf{0.16} & 0.17 & 0.21 & \textbf{0.16} & \textbf{0.16} & 0.21 & \textbf{0.15} & \textbf{0.15} \\
\hallwaytwo & 1.77 & \textbf{1.53} & N/A & 1.66 & \textbf{1.47} & 1.57 & 1.56 & \textbf{1.35} & 1.38 \\
\pentagon & 0.19 & \textbf{0.14} & N/A & 0.15 & \textbf{0.10} & 0.15 & 0.12 & \textbf{0.08} & 0.12 \\
\fourth & \textbf{0.66} & 0.92 & N/A & \textbf{0.57} & 0.62 & N/A & 0.44 & \textbf{0.41} & N/A \\
\bottomrule
\end{tabular}

%% file: Tables/Eval_Sarsop_Small.tex
\begin{tabular}{lcrcrcr}
\toprule
\textbf{Environments} && \multicolumn{1}{c}{\FIB} && \BIB && \EBIB \\
\cmidrule{1-1} \cmidrule{3-3} \cmidrule{5-5} \cmidrule{7-7}
\custom & & <1s & & <1s & & <1s \\
\tiger & & <1s & & <1s & & <1s \\
\rocksample (5,3) & & <1s & & <1s & & <1s \\
\koutofn (2) & & 616s & & 106s & & 119s \\
\aloha (30) & & 45s & & 40s & & 945s \\
\bottomrule
\end{tabular}

%% file: Tables/Eval_Bounds_CTIB.tex
\begin{tabular}{lrcrrcrrcr}
\toprule
 \textbf{Environments} && \multicolumn{2}{c}{\BIB} && \multicolumn{2}{c}{CTIB} && \multicolumn{2}{c}{\EBIB} \\
\cmidrule{1-1} \cmidrule{3-4} \cmidrule{6-7} \cmidrule{9-10}
\custom &  & 0.61 & <1s &  & \textbf{0.50} & <1s &  & \textbf{0.50} & <1s \\
\tiger &  & 49.6 & <1s &  & 48.6 & <1s &  & \textbf{40.5} & <1s \\
\gridenv &  & 8.15 & 4s &  & 7.41 & 6s &  & \textbf{7.25} & 5s \\
\rocksample (5,3) &  & \textbf{18.3} & <1s &  & \textbf{18.3} & <1s &  & \textbf{18.3} & <1s \\
\rocksample (7,8) &  & \textbf{27.2} & 413s &  & \textbf{27.2} & 717s &  & \textbf{27.2} & 728s \\
\koutofn (2) &  & \textbf{-1.52} & 4s &  & -1.47 & 6s &  & \textbf{-1.52} & 4s \\
\koutofn (3) &  & \textbf{-2.28} & 9s &  & -2.22 & 16s &  & \textbf{-2.28} & 14s \\
\aloha (30) &  & \textbf{392} & 24s &  & \textbf{392} & 271s &  & \textbf{392} & 923s \\
\tagenv &  & \textbf{-5.58} & 22s &  & -5.42 & 36s &  & -5.57 & 35s \\
\tigergrid &  & 2.58 & 31s &  & 2.58 & 90s &  & \textbf{2.57} & 270s \\
\hallwayone &  & 1.19 & 28s &  & 1.18 & 103s &  & \textbf{1.17} & 398s \\
\hallwaytwo &  & 0.89 & 52s &  & 0.89 & 208s &  & \textbf{0.88} & 777s \\
\pentagon &  & \textbf{0.38} & 24s &  & \textbf{0.38} & 250s &  & \textbf{0.38} & 645s \\
\fourth &  & \textbf{0.09} & 432s &  & \textbf{0.09} & TO &  & \textbf{0.09} & TO \\
\bottomrule
\end{tabular}

%% file: Tables/Sarsop_Large_full.tex
\begin{tabular}{lcrrrcrrrcrrr}
\toprule
&& \multicolumn{3}{c}{\textbf{600s}} && \multicolumn{3}{c}{\textbf{1200s}} && \multicolumn{3}{c}{\textbf{3600s}}\\
\cmidrule{3-5} \cmidrule{7-9} \cmidrule{11-13}
\textbf{Environments} && \multicolumn{1}{c}{\FIB} & \multicolumn{1}{c}{\BIB} & \multicolumn{1}{c}{\EBIB} && \multicolumn{1}{c}{\FIB} & \multicolumn{1}{c}{\BIB} & \multicolumn{1}{c}{\EBIB} && \multicolumn{1}{c}{\FIB} & \multicolumn{1}{c}{\BIB} & \multicolumn{1}{c}{\EBIB} \\
\cmidrule{1-1} \cmidrule{3-5} \cmidrule{7-9} \cmidrule{11-13}
\gridenv & & 6.42, 7.50 & 6.41, 7.42 & 6.41, 6.89 & & 6.42, 7.29 & 6.42, 7.24 & 6.42, 6.84 & & 6.42, 7.04 & 6.42, 6.99 & 6.42, 6.76 \\
\rocksample (7,8) & & 20.45, 25.45 & 20.16, 25.68 & 19.65, 25.71 & & 20.92, 25.25 & 20.75, 25.27 & 21.09, 25.38 & & 21.93, 25.01 & 21.81, 24.98 & 21.65, 25.03 \\
\koutofn (3) & & -2.63, -2.07 & -2.63, -2.33 & -2.63, -2.32 & & -2.63, -2.10 & -2.63, -2.35 & -2.63, -2.34 & & -2.63, -2.14 & -2.63, -2.37 & -2.63, -2.37 \\
\tagenv & & -10.86, -6.05 & -10.86, -6.44 & -10.87, -6.45 & & -10.85, -6.13 & -10.84, -6.51 & -10.84, -6.52 & & -10.82, -6.29 & -10.82, -6.69 & -10.82, -6.69 \\
\tigergrid & & 2.28, 2.53 & 2.26, 2.47 & 2.24, 2.47 & & 2.28, 2.52 & 2.28, 2.47 & 2.27, 2.46 & & 2.29, 2.52 & 2.28, 2.46 & 2.28, 2.46 \\
\hallwayone & & 0.99, 1.21 & 0.99, 1.15 & 0.97, 1.15 & & 1.00, 1.21 & 0.99, 1.15 & 0.99, 1.15 & & 1.00, 1.21 & 1.00, 1.15 & 1.00, 1.15 \\
\hallwaytwo & & 0.32, 0.91 & 0.34, 0.86 & N/A & & 0.33, 0.91 & 0.34, 0.85 & 0.33, 0.86 & & 0.35, 0.90 & 0.36, 0.85 & 0.35, 0.85 \\
\pentagon & & 0.32, 0.38 & 0.33, 0.38 & N/A & & 0.33, 0.38 & 0.34, 0.38 & 0.32, 0.38 & & 0.34, 0.38 & 0.35, 0.38 & 0.34, 0.38 \\
\fourth & & 0.05, 0.09 & 0.04, 0.09 & N/A & & 0.05, 0.09 & 0.05, 0.09 & N/A & & 0.06, 0.09 & 0.06, 0.09 & N/A \\
\bottomrule
\end{tabular}